\documentclass[lettersize,journal]{IEEEtran}
\usepackage{amsmath,amsfonts}
\usepackage{algorithmic}
\usepackage{algorithm}
\usepackage{array}
\usepackage[caption=false,font=footnotesize,labelfont=rm,textfont=rm]{subfig}
\usepackage{textcomp}
\usepackage{stfloats}
\usepackage{url}
\usepackage{verbatim}
\usepackage{graphicx}
\usepackage{cite}
\usepackage{multirow}
\usepackage{epsfig}
\usepackage{amsthm}
\newtheorem{Definition}{Definition}

\newtheorem{theorem}{Theorem}

\newtheorem{proposition}{Proposition}
\usepackage[accsupp]{axessibility}
\usepackage{booktabs}

\begin{document}

\title{Unifying Perplexing Behaviors in Modified BP Attributions through Alignment Perspective}

\author{Guanhua~Zheng,
        Jitao~Sang, 
        and~Changsheng~Xu, ~\IEEEmembership{Fellow,~IEEE,}
\thanks{Manuscript received XXX. The Associate Editor coordinating the review of this manuscript and approving it for publication was XXX. (Corresponding author: Changsheng Xu.)}
\IEEEcompsocitemizethanks{\IEEEcompsocthanksitem G. Zheng is with the University of Science and Technology of China, Hefei 230026, China (e-mail: zhenggh@mail.ustc.edu.cn).\protect\\
\IEEEcompsocthanksitem J. Sang is with the School of Computer and Information Technology and the Beijing Key Laboratory of Traffic Data Analysis and Mining, Beijing Jiaotong University, Beijing 100044, China (e-mail: jtsang@bjtu.edu.cn).\protect\\
\IEEEcompsocthanksitem C. Xu is with the National Lab of Pattern Recognition, Institute of Automation, CAS, Beijing 100190, China, and the University of Chinese Academy of Sciences (e-mail: csxu@nlpr.ia.ac.cn).}
}

\markboth{IEEE XXXX}%
{Shell \MakeLowercase{\textit{et al.}}: A Sample Article Using IEEEtran.cls for IEEE Journals}

\IEEEpubid{0000--0000/00\$00.00~\copyright~2021 IEEE}

\maketitle

\begin{abstract}
Attributions aim to identify input pixels that are relevant to the decision-making process. A popular approach involves using modified backpropagation (BP) rules to reverse decisions, which improves interpretability compared to the original gradients. However, these methods lack a solid theoretical foundation and exhibit perplexing behaviors, such as reduced sensitivity to parameter randomization, raising concerns about their reliability and highlighting the need for theoretical justification.
In this work, we present a unified theoretical framework for methods like GBP, RectGrad, LRP, and DTD, demonstrating that they achieve input alignment by combining the weights of activated neurons. This alignment improves the visualization quality and reduces sensitivity to weight randomization. Our contributions include: (1) Providing a unified explanation for multiple behaviors, rather than focusing on just one. (2) Accurately predicting novel behaviors. (3) Offering insights into decision-making processes, including layer-wise information changes and the relationship between attributions and model decisions.
\end{abstract}


\section{Introduction}
\IEEEPARstart{D}{eep} learning has long grappled with its inherent black-box nature\cite{blackbox}, which poses challenges to users seeking insight into the internal workings of deep models. This opacity exposes users to the risk of a shortcut decision\cite{geirhos2020shortcut}, impedes applications in high-risk areas \cite{survey_tjoa}, and underscores the imperative for eXplainable Artificial Intelligence (XAI)\cite{survey_zhang}. Attribution methods, also known as saliency methods, strive to illuminate the most relevant image regions, enabling users to understand the decision-making process of deep models. One widely embraced attribution technique involves backpropagating (BP) the relevance score from the output to the input pixels. Recently, there has been a proliferation of modified BP methods that reveal semantically meaningful patterns, such as edges outlining target objects. These studies also report excellent performance on occlusion-type measures of model faithfulness\cite{evaluating_samek,Agarwal2020ExplainingIC,ijcai2020p417} that remove high-score pixels, indeed changing model predictions significantly. Furthermore, modified BP methods have demonstrated success in detecting spurious correlations\cite{lapuschkin2019unmasking, Rieger2020useful}. These impactful visualizations and behaviors foster the belief that such attributions shed light on the internal mechanisms of model decision-making. 

\IEEEpubidadjcol
However, several modified BP attribution methods have faced scrutiny for behaviors misaligned with their intended functionality, particularly their insensitivity to weight changes during sanity checks \cite{adebayo2018sanity}. Specifically, randomizing model weights can drastically alter its decisions, yet modified BP methods often produce nearly identical visualizations. This observation leads Adebayo et al. \cite{adebayo2018sanity} to question the real role of attribution explanations, noting that even an edge detector, independent of the model or training data, can generate outputs visually similar to saliency maps.

This highlights the urgent need for theoretical support to address such criticisms. 
Some existing theories have provided partial explanations for the aforementioned perplexing behaviors\cite{nie2018theoretical, sixt2020explanations, shortcoming}, but they still fall far short of the necessary scope. First, these theories exhibit significant limitations, as they are only applicable to certain attribution methods under specific conditions. For example, \cite{nie2018theoretical} is only applicable to GBP \cite{springenberg2014striving} and cannot explain LRP \cite{bach2015pixel}, while \cite{sixt2020explanations} does the opposite. However, the explanatory results and behaviors of GBP and LRP are very similar (see Table \ref{sample}), suggesting that there may be a unified theoretical framework underlying these methods. Second, these theories are limited in that they can only describe existing behaviors and fail to predict new ones. This raises concerns about their capacity for future extensions, as these theories may not be able to explain new attribution behaviors. Lastly, the practical applicability of these theories is limited. Based on their conclusions, researchers can only understand why certain attributions exhibit perplexing behaviors, but they cannot leverage the theories to obtain clearer, more actionable insights into the decision-making process.

In this paper, we identify the core reason behind the perplexing behaviors of modified BP attributions: attributions represent the input information used in decision-making, rather than a score of the inputs. Specifically, we demonstrate that attribution methods such as GBP\cite{springenberg2014striving}, RectGrad\cite{kim2019saliency}, LRP\cite{bach2015pixel}, and DTD\cite{montavon2017explaining}, which modify the backpropagation rules, belong to a special category we define as the Negative Filtering Rule (NFR). These rules guide the flow of input information during backpropagation through a process of cascade alignment, involving only the weights corresponding to activated neurons. As a result, the final attribution outcome can be interpreted as reflecting the input information that the model uses for decision-making.

Our theoretical framework provides an effective explanation for these perplexing behaviors. On one hand, the model inputs are typically understandable to humans, and since these attribution methods share the same class of backpropagation mechanisms, they naturally exhibit similar behaviors. On the other hand, our theory asserts that weight randomization does not disrupt the cascade alignment process, rendering attribution results insensitive to weight randomizations. To demonstrate the effectiveness of our theory in explaining attribution behaviors, we designed experiments to validate the key conclusions and conditions outlined in the theory. Furthermore, we compare our theoretical predictions with existing theories, showing that our predictions are more consistent with experimental results. Finally, we show the new insights brought about by our theory to the use of attribution methods.

In summary  our contribution can be outlined in three main aspects:
\begin{itemize} 
\item \textbf{Unifying perplexing behaviors in modified BP attributions through input alignment.}
We introduce a novel perspective on modified BP attribution methods, demonstrating that these methods essentially align with input during the modified backpropagation. This alignment explains why the attribution results are interpretable for humans, perform well in occlusion-type reliability metrics, and are insensitive to weight randomization (Section \ref{sec:theory}).
\item \textbf{Predicting novel behaviors.} The results and conditions of our theory are empirically validated through experiments. In addition, we analyze the differences in which phenomena are predicted by our theory and existing theories, and the experimental results showed that our predictions are more accurate (Section \ref{sec:experiment}).
\item \textbf{Enhancing the capability of attributions with new insights.} In this work, we emphasize the enhancement of attribution capabilities by considering attributions as key input information in decision-making. Specifically, we compare layer-wise alignments to track how input information evolves throughout the decision-making process. Additionally, by feeding attribution results back into the model, we assess the influence of input information on the decision-making process. This approach helps establish a connection between attribution results and the model's generalizability. It also demonstrates the impact of the generation method of backdoor attacks on the role of the backdoor in the decision-making process. (Section \ref{sec:application}).
\end{itemize}

\section{Related Work}
\textbf{Theories for attribution behaviors.} Modified BP attribution methods have gained attention due to concerns about the opaque nature of deep learning models. However, these methods often lack a solid theoretical foundation, prompting efforts to analyze and explain their behavior. Ancona et al. \cite{ancona2017towards} establish the equivalence between Layer-wise Relevance Propagation(LRP) \cite{bach2015pixel} and Gradient $\odot$ Input\cite{simonyan2013deep}. Rebuffi et al. \cite{rebuffi2020there} propose the Extract-Aggregate framework, which explains why GradCAM \cite{selvaraju2017grad} generates meaningful results only at the final convolutional layer. Deng et al. \cite{deng2020unified} introduce a unified Taylor framework, showing that backpropagation-based attributions can be understood through various-order approximations. These theories offer some insights into the function of attribution explanations, but they generally provide only a broad understanding, limiting their ability to justify specific model behaviors.

A key challenge with modified backpropagation (BP) attribution methods lies in their ability to produce interpretable visualizations but are insensitive to weight randomizations (sanity checks) \cite{adebayo2018sanity}. 
Subsequent works have provided constructive suggestions to the original sanity checks \cite{adebayo2018sanity}. These include the choices for explanation preprocessing\cite{local_sanity}, the dependency on the model task \cite{revisiting_sanity,invesgation_sanity}, and the use of sampling to mitigate noise effects \cite{exploration_sanity}. However, these studies do not provide a theoretical analysis of the effectiveness of these approaches.
To theoretical justify the perplexing behavior, Nie et al. \cite{nie2018theoretical} argue that GBP \cite{springenberg2014striving} partially recovers the input, while Sixt et al. \cite{sixt2020explanations} show that the $z^+$-rule used in DTD \cite{montavon2017explaining}, LRP \cite{bach2015pixel}, and Excitation BP \cite{zhang2018top} leads to a multiplication chain of nonnegative matrices, which converging to a rank-1 matrix.  Binder et al. \cite{shortcoming} highlight inconsistencies in the ranking of attribution methods across randomization-based sanity checks and occlusion-type faithfulness measures, demonstrating that top-down randomization preserves the activation scale with high probability, maintaining strong contributions from large activations.

Existing theoretical frameworks primarily explain individual behaviors but fail to offer a unified justification for the range of attribution behaviors. They lack sufficient theoretical guarantees for modified BP attributions. In contrast, our theory provides stronger support by: (1) unifying explanations across multiple behaviors, (2) accurately predicting new attribution behaviors, and (3) offering practical guidance for interpreting model decisions effectively.

\textbf{Alignment Hypothesis}: The recent focus on the phenomenon of robust models yielding structured gradients has been underscored in literature \cite{tsipras2018robustness}. Two predominant theoretical frameworks currently explain this phenomenon: Alignment with input images \cite{etmann2019connection}, and alignment with image manifolds \cite{kim2019bridging}. These explanations effectively delineate the distinction between robust and raw model visualizations. We leverage these insights to explore why modifications to the backpropagation rule yield semantically meaningful explanations and to identify key factors that affect the diversity of visualizations. Our conception of alignment use the definition proposed by \cite{etmann2019connection}.
\begin{Definition}
(Input Alignment): Let  $\textbf{x}$ be an input image, and $R$ be the corresponding saliency map of target CNN models, we call 
\begin{equation}
    \alpha(R,\textbf{x})=\frac{\langle R , \textbf{x}\rangle}{\|R\|\|\textbf{x}\|}
\end{equation}
the alignment between $R$ and $\textbf{x}$, where $\langle\cdot ,\cdot \rangle$ is the standard inner product, $\|\cdot\|$ is the $L_2$ norm.
\end{Definition}

\section{Justification of Perplexing Behavior with Alignment Perspective}
\label{sec:theory}
In this section, we present a theory of the modified BP attribution from an input alignment perspective. Section \ref{sec3.1} reviews four existing methods: GBP\cite{zeiler2014visualizing}, RectGrad\cite{kim2019saliency}, LRP\cite{bach2015pixel}, and DTD\cite{montavon2017explaining}. In Section \ref{sec3.2}, we demonstrate that the core backpropagation rule underlying these methods can be unified as a Negative Filtering Rule (NFR). NFR aligns with the input through the backpropagation process, capturing the key information in inputs for decision-making. Finally, Section \ref{sec3.3} explains why modified BP attribution produces visually interpretable, occlusion-type reliable explanations, while remaining insensitive to the parameter randomization.

\subsection{Preliminary}
\label{sec3.1}
Let us begin by introducing key notation for a multicategory classification problem with $K$ classes. Suppose that we have an image of d-dimensional $\textbf{x}=\{x_1,...,x_d\}\in \mathbb{R}^d$ and its label is $y\in\{1,...,K\}$. We feed $\textbf{x}$ into a deep neural network function $f:\mathbb{R}^d \mapsto \mathbb{R}^K$. The attribution map for the $k^{th}$ class is denoted as $R_{k}: \mathbb{R}^d \mapsto \mathbb{R}^d$. Specifically, $W_l$ represents the weight of the $l^{th}$ layer of $f$, $\textbf{A}_l$ is the activation vector for the $l^{th}$ layer, and the deep neural network $f$ consists of $L$ layers. The forward propagation is expressed as follows:
\begin{equation}
\textbf{A}_{l}=\sigma(W_{l}^{T}\textbf{A}_{l-1})
\end{equation}
where $\textbf{A}_{0}=\textbf{x}$, and the ReLU function is a widely used activation function $\sigma(x)=(x,0)_{+}$. Just like previous theoretical studies\cite{sixt2020explanations}\cite{nie2018theoretical}, we ignore the bias terms because the BP-based attributions mentioned in our work do not use the biases. Now we can obtain the network output $f$:
\begin{equation}
f=W_{L}^{T}\textbf{A}_{L-1}
=W_{L}^{T}(\prod_{l=1}^{L-1}{M_{(L-l)}W_{(L-l)}^{T}})\textbf{x}
\label{eq2}
\end{equation}
where $M_l=diag(\mathbb{I}(W_{l}^{T}\textbf{A}_{l-1}))$ denotes the gradient mask of the ReLU operation, $\mathbb{I}(\cdot)$ is the indicator function output 1 if input greater than 0, else output 0.

Therefore, the original BP-based attribution of k-th output is:
\begin{equation}
R_k=\frac{\partial f_k}{\partial \textbf{x}}=(\prod_{l=1}^{L-1}{W_{l}M_l})\textbf{v}_k
\end{equation}
where $\textbf{v}_k$ is $k^{th}$ row of $W_{L}^{T}$ which can produce the $k^{th}$ output. 

To elucidate the intricate adjustments in the backpropagation rules, we introduce $r_l^{\Delta}$ as the relevance propagating to the intermediate layer $\textbf{A}_{l}$, with $r_l^{\Delta}=\frac{\partial^{\Delta} f}{\partial^{\Delta} \textbf{A}_{l}}$, and $R^{\Delta}$ denotes the ultimate results. In the context of Grad, GBP \cite{springenberg2014striving}, RectGrad \cite{kim2019saliency}, LRP \cite{bach2015pixel}, and DTD \cite{montavon2017explaining}, we employ $\Delta={\_, g, r, l, d}$, representing different methodologies. 

\textbf{Grad:} Grad is the original backpropagation rule:
\begin{equation}
r_{l-1}=W_l M_l r_{l}
\end{equation}
and the final results $R=r_0$. In practice, researchers typically point-wise muiltiply the input $R=r_0\odot \textbf{x}$.

\textbf{GBP:} Guided BackPropagation masks the negative relevance at the ReLU layer:
\begin{equation}
r_{l-1}^g=W_l M_l \sigma(r^g_{l})
\end{equation}
and the final results $R^g=r_0^g$.

\textbf{RectGrad}: RectGrad sets a threshold $\tau$ to filter the low contribution connection at ReLU layer:
\begin{equation}
r_{l-1}^r=W_l \left(r_{l}^r\odot\mathbb{I}(\textbf{A}_l \odot r_{l}^r>\tau) \right )
\end{equation}
where $\odot$ is a point-wise multiple. Different from GBP and Grad, the final result of RectGrad has post-process: $R^r=\sigma(\textbf{x}\cdot r_0^r)$. We use $\tau=90\%$ in the following experiments. 

\textbf{LRP:}  LRP redistributes the relevance of $l^{th}$ layer to $(l-1)^{th}$ layer according to the weighted activations $Z_l=\{z_{l_{[ij]}}\}=W_{l}^T\textbf{A}_{l-1}$, and just pass the ReLU layers:
\begin{equation}
r_{(l-1)_{[i]}}^{z}=\sum_j{\left (\frac{z_{l_{[ij]}}}{\sum_k{z_{l_{[kj]}}} } \right )r_{l_{[j]}}^{z}}
\end{equation}
where $z_{l_{[ij]}}$ denotes the forward propagation from the $i^{th}$ item in $(l-1)^{th}$ layer ($a_{(l-1)_{[i]}}$) to the $j^{th}$ item in $l^{th}$ layer ($a_{l_{[j]}}$).

As such changes still result in noisy attributions, LRP proposes to separate the positive and negative influences:
\begin{equation}
r_{(l-1)_{[i]}}^{l_{\alpha\beta}}=\sum_j{\left (\alpha\frac{z_{l_{[ij]}}^+}{\sum_k{z_{l_{[kj]}}^+}}-\beta\frac{z_{l_{[ij]}}^-}{\sum_k{z_{l_{[kj]}}^-}} \right )r_{l_{[j]}}^{l_{\alpha\beta}}}
\end{equation}

where $z^+=(z,0)_+$, $z^-=(z,0)_-$. In this paper, we concentrate on $\alpha=1, \beta=0$, which have excellent visualization performance, $LRP_{\alpha1\beta0}$ can also be formalized as $z^+$ rule:
\begin{equation}
r_{l_{[i]}}^{+}=\sum_j{\left (\frac{z_{l_{[ij]}}^+}{\sum_k{z_{l_{[kj]}}^+}} \right )r_{(l+1)_{[j]}}^{+}}
\end{equation}
where $z^+=(z,0)_+$. \emph{All the LRP mentioned below means} $LRP_{\alpha1\beta0}$.

\textbf{DTD:} Deep Taylor Decomposition(DTD) has two types of modification of BP-rules: $z^+$-rule and $z^{\mathcal{B}}$-rule. If the inputs of corresponding layer are in $[0,\infty]$, DTD uses $z^+$-rule which is equivalent to $LRP_{\alpha1\beta0}$. Else if the layer inputs are bounded in $[l_i,h_i]$, like the input images, DTD uses $z^{\mathcal{B}}$-rule:
\begin{equation}
r_{(l-1)_{[i]}}^{d}=\sum_j{\left (\frac{z_{l_{[ij]}}-l_i\omega_{l_{[ij]}}^+-h_i\omega_{l_{[ij]}}^-}{\sum_k{(z_{l_{[kj]}}-l_i\omega_{l_{[kj]}}^+-h_i\omega_{l_{[kj]}}^-} )} \right )r_{l_{[j]}}^{d}}
\end{equation}
where $\omega_{l_{[ij]}}^+=(\omega_{l_{[ij]}},0)_+$, and $\omega_{l_{[ij]}}^-=(\omega_{l_{[ij]}},0)_-$.

\setlength{\tabcolsep}{4pt}
\begin{table*}[t]
\caption{Mathematical formulation and examples of ImageNet. The attribution can be partitioned two parts: the main BP rule backpropagate predictions from output to the input $r_0$ (we show the \textbf{filter} of NFR for brief), and the \textbf{bottom} process which obtain the final results $R$ with $r_0$. The \textbf{R w/o B} and \textbf{T w/o B} means use random or pretrained weights without bottom process. }
\label{sample}
\begin{center}
\begin{small}
\begin{sc}
\begin{tabular}{lcccccccr}
\toprule
Method & Grad$\odot$Input & GBP & RectGrad & LRP  & DTD & Activation \\
\midrule
Filter & -& $\mathbb{I}(r^g_{l}>0)$&  $ \mathbb{I}(\textbf{A}_l \odot r_{l}^r>\tau)$ & $\mathbb{I}(W_{l}>0) $ & $\mathbb{I}(W_{l}>0)$ & $\textbf{A}_l$\\
Bottom & $r_0\odot \textbf{x}$ & $r_0^g$ & $\sigma(r_0^r\odot \textbf{x})$ & $ r_0^l\odot \textbf{x}$& $z^{\mathcal{B}}r_1^d\odot \textbf{x}$& $r_0^a$\\
\midrule
Random  & \begin{minipage}{0.12\textwidth}
      \includegraphics[width=1\textwidth]{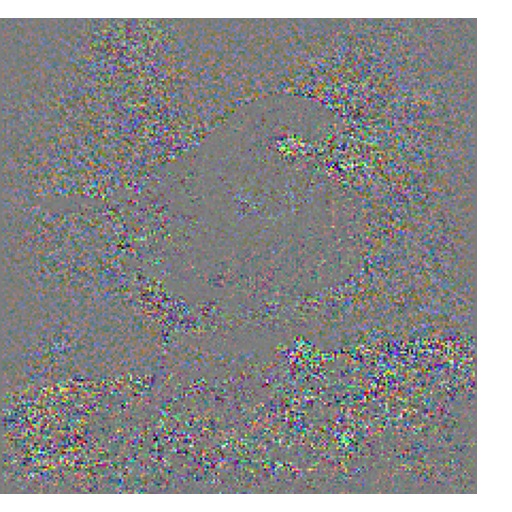}
    \end{minipage} & \begin{minipage}{0.12\textwidth}
      \includegraphics[width=1\textwidth]{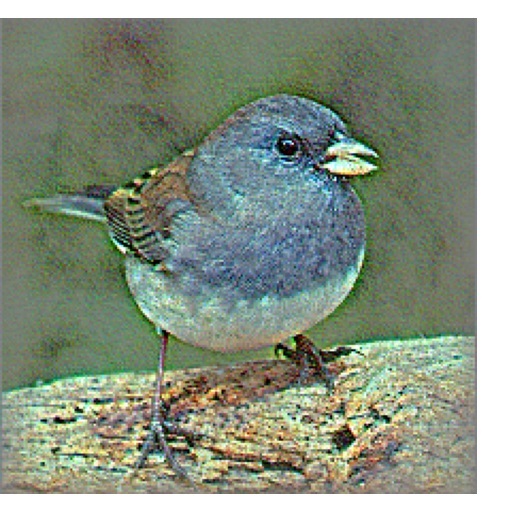}
    \end{minipage}& \begin{minipage}{0.12\textwidth}
      \includegraphics[width=1\textwidth]{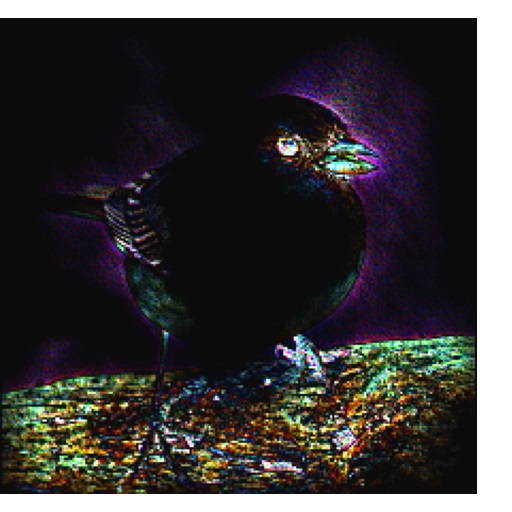}
    \end{minipage} & \begin{minipage}{0.12\textwidth}
      \includegraphics[width=1\textwidth]{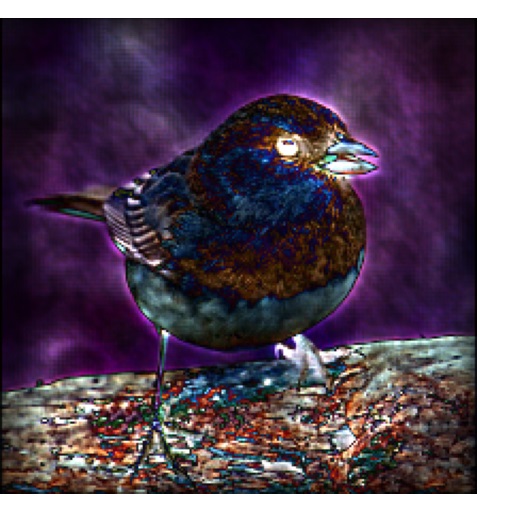}
    \end{minipage}& \begin{minipage}{0.12\textwidth}
      \includegraphics[width=1\textwidth]{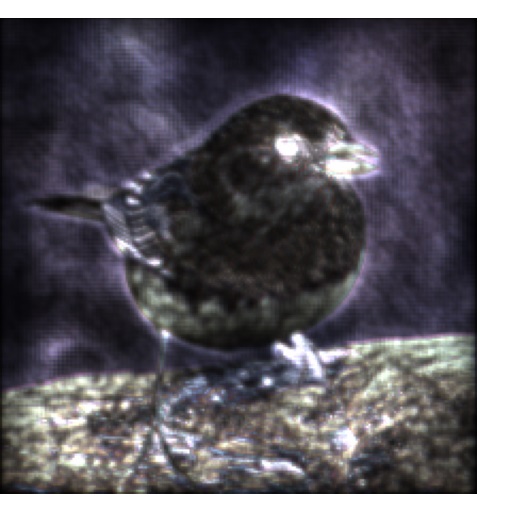}
    \end{minipage}& \begin{minipage}{0.113\textwidth}
      \includegraphics[width=1\textwidth]{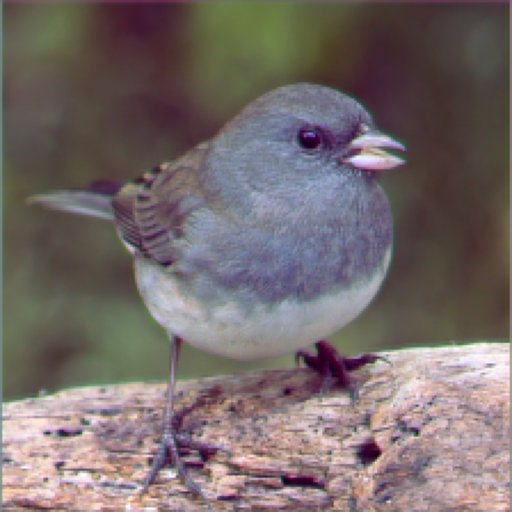}
    \end{minipage}\\
Trained &\begin{minipage}{0.12\textwidth}
      \includegraphics[width=1\textwidth]{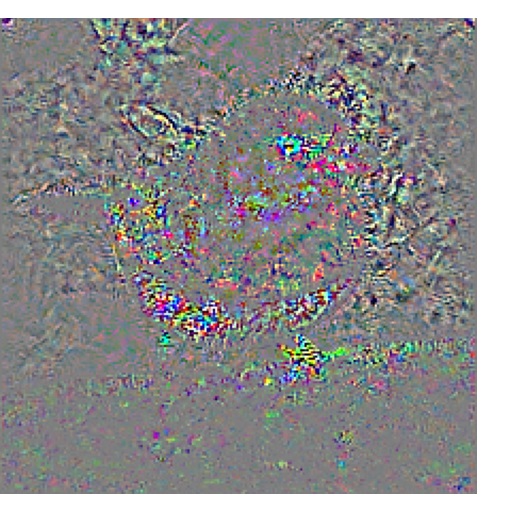}
    \end{minipage} & \begin{minipage}{0.12\textwidth}
      \includegraphics[width=1\textwidth]{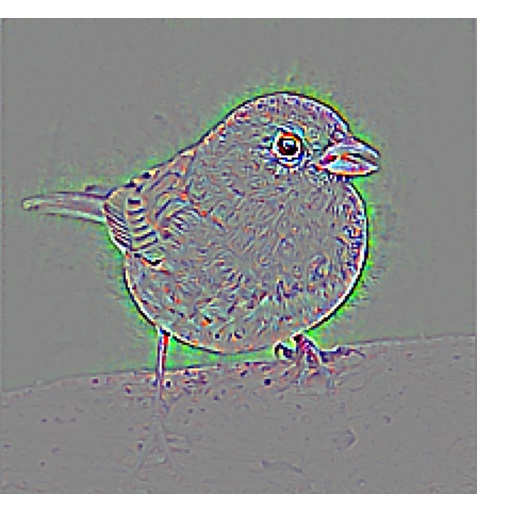}
    \end{minipage}& \begin{minipage}{0.12\textwidth}
      \includegraphics[width=1\textwidth]{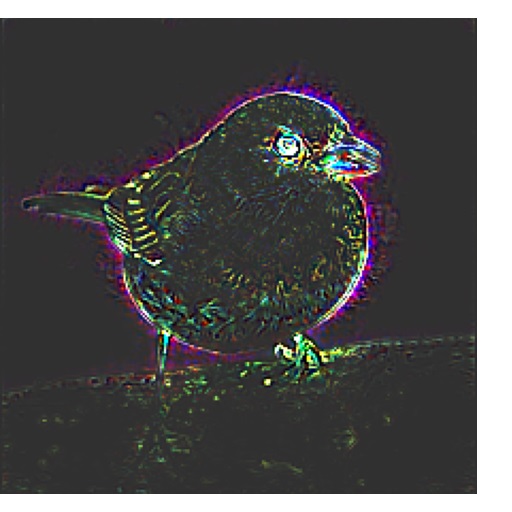}
    \end{minipage} & \begin{minipage}{0.12\textwidth}
      \includegraphics[width=1\textwidth]{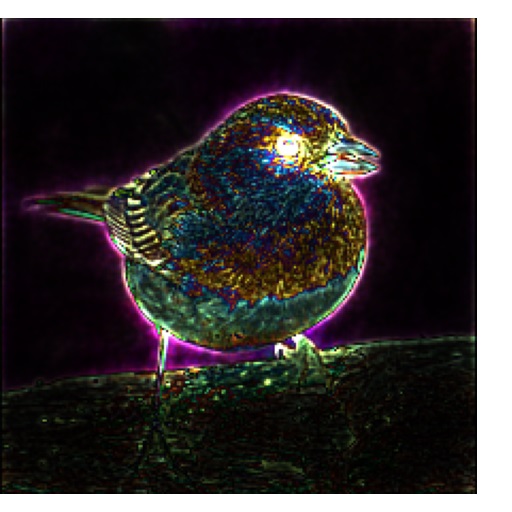}
    \end{minipage}& \begin{minipage}{0.12\textwidth}
      \includegraphics[width=1\textwidth]{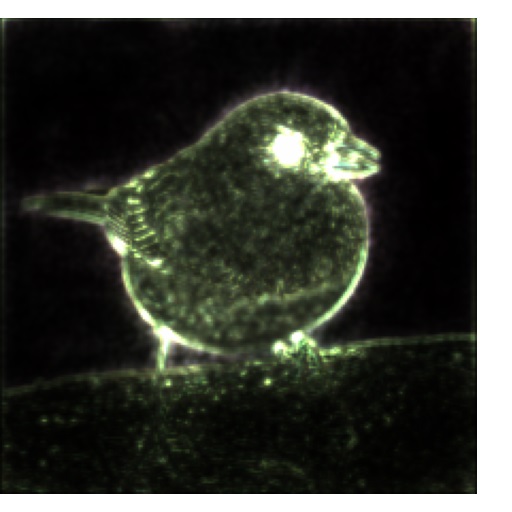}
    \end{minipage}& \begin{minipage}{0.113\textwidth}
      \includegraphics[width=1\textwidth]{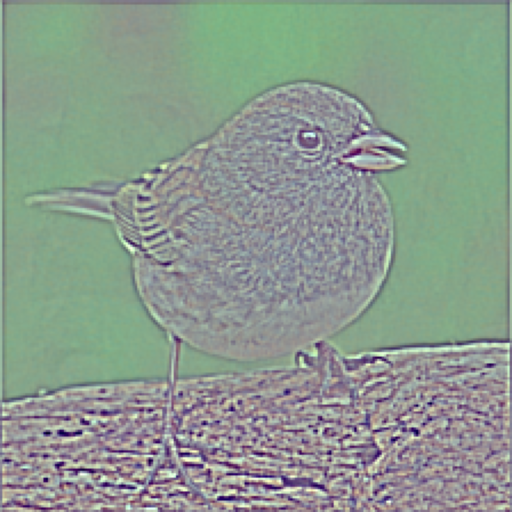}
    \end{minipage}\\
\midrule
R w/o B& \begin{minipage}{0.12\textwidth}
      \includegraphics[width=1\textwidth]{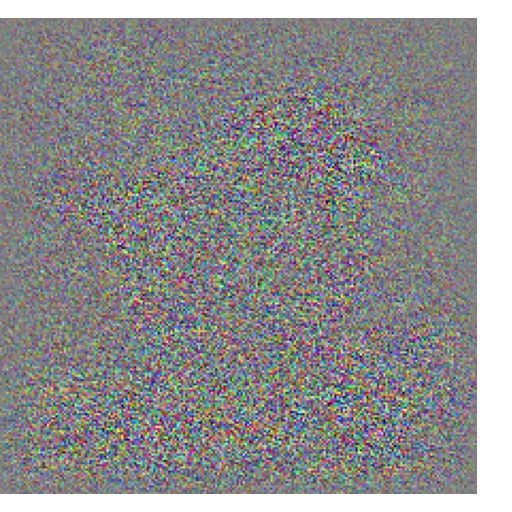}
    \end{minipage} & \begin{minipage}{0.12\textwidth}
      \includegraphics[width=1\textwidth]{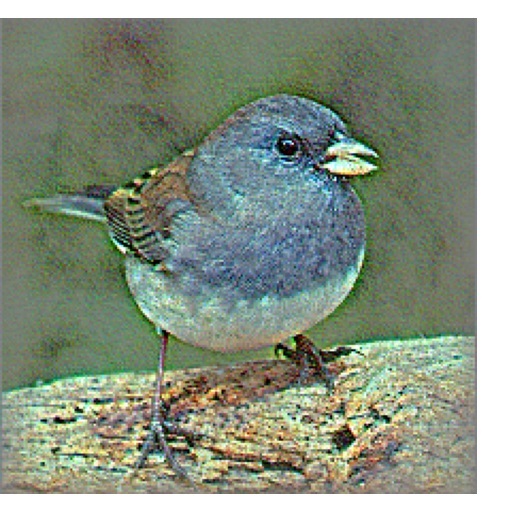}
    \end{minipage}& \begin{minipage}{0.12\textwidth}
      \includegraphics[width=1\textwidth]{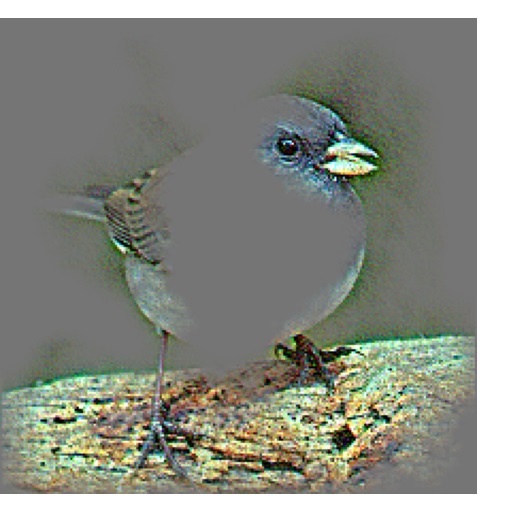}
    \end{minipage} & \begin{minipage}{0.12\textwidth}
      \includegraphics[width=1\textwidth]{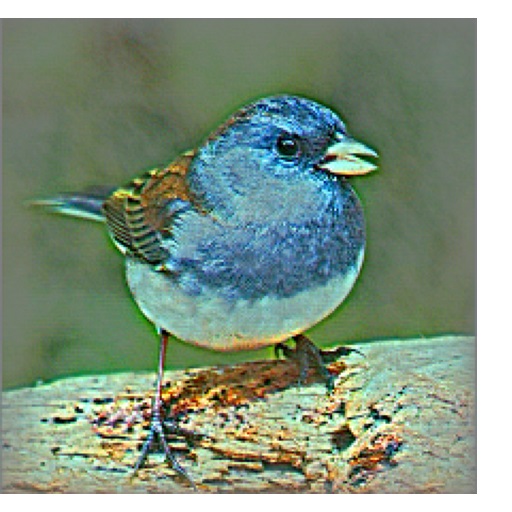}
    \end{minipage}& \begin{minipage}{0.12\textwidth}
      \includegraphics[width=1\textwidth]{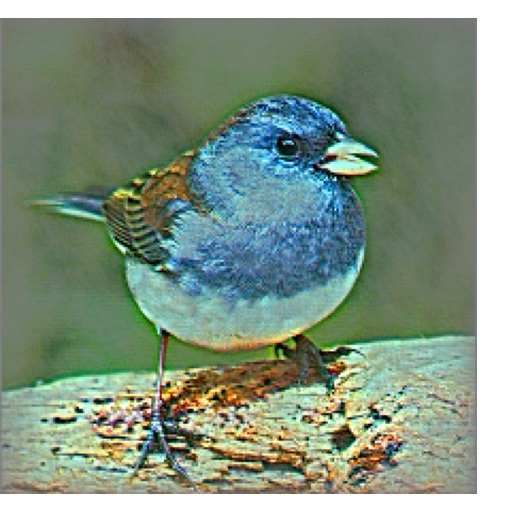}
    \end{minipage}& \begin{minipage}{0.113\textwidth}
      \includegraphics[width=1\textwidth]{images/Our_random.png}
    \end{minipage}\\
T w/o B   & \begin{minipage}{0.12\textwidth}
      \includegraphics[width=1\textwidth]{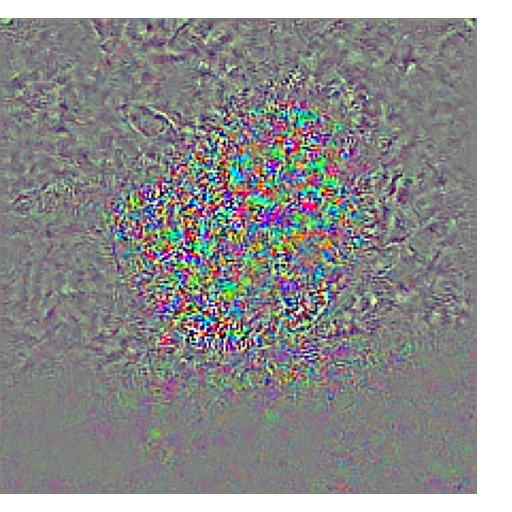}
    \end{minipage} & \begin{minipage}{0.12\textwidth}
      \includegraphics[width=1\textwidth]{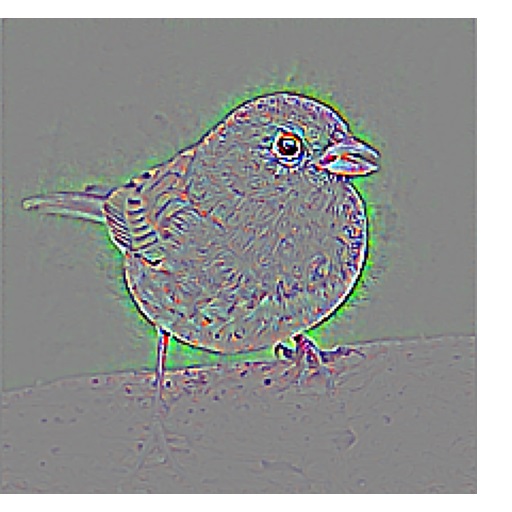}
    \end{minipage}& \begin{minipage}{0.12\textwidth}
      \includegraphics[width=1\textwidth]{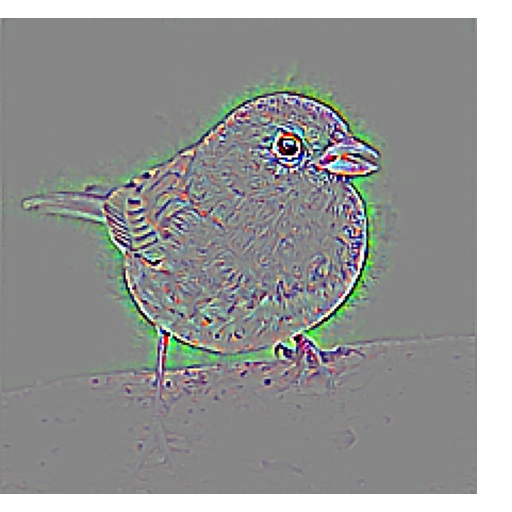}
    \end{minipage} & \begin{minipage}{0.12\textwidth}
      \includegraphics[width=1\textwidth]{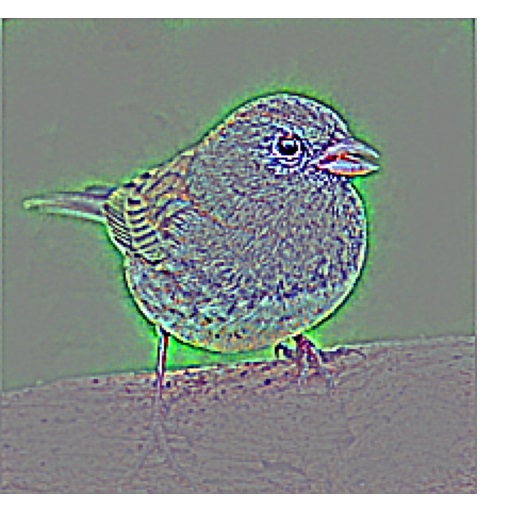}
    \end{minipage}& \begin{minipage}{0.12\textwidth}
      \includegraphics[width=1\textwidth]{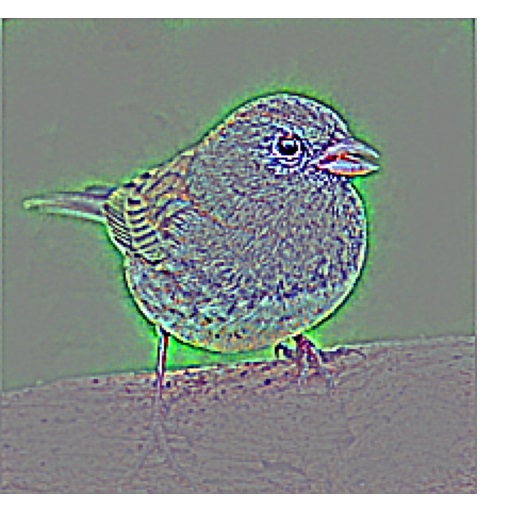}
    \end{minipage}& \begin{minipage}{0.113\textwidth}
      \includegraphics[width=1\textwidth]{images/Our_pretrain.png}
    \end{minipage}\\

\bottomrule
\end{tabular}
\end{sc}
\end{small}
\end{center}
\vskip -0.1in
\end{table*}

\subsection{Alignment between Modified BP Attributions and the Input}
\label{sec3.2}

To ensure clarity in our analysis, we introduce \emph{Negative Filtering Rule}(NFR) to unified formalize these methods. The NFR means a modified rule that selectively masks certain items during the backpropagation process and can be formalized as follows:
\begin{Definition}
(Filtering Rule and Negative Filtering Rule) The filtering rule is the modified rule that inserts a filter $F_l$ in the backpropagation rule of raw grad:
\begin{equation}
    r_{l-1}=F_l(W_l M_l) r_{l}
\end{equation}
where $F_l$ is a transformation that zeros out some of the terms in $W_l$ or $M_l$ by inserting $diag(\mathbb{I}(\cdot))$.

Negative filtering rule, adding an additional constraint that
\begin{equation}
    \langle \textbf{A}_{(l-1)}, F_l(W_l M_l) r_{l} \rangle> \langle \textbf{A}_{(l-1)}, W_l M_l r_l\rangle
\end{equation}
when $F_l(W_l M_l)\neq W_l M_l$.
\end{Definition}

Next, we will demonstrate that the main rules of GBP, RectGrad, LRP, and DTD can be converted into a negative filtering rule (NFR).

\begin{proposition}
the backpropagation rule of GBP:
\begin{equation}
r_{l-1}^g=W_l M_l \sigma(r^g_{l})
\end{equation}
can be formalized as a filtering rule for $F^g_l(W_l M_l)=W_l M_l diag(\mathbb{I}(r^g_{l}>0))$, and is negative filtering rule.
\end{proposition}

\begin{proposition}
the backpropagation rule of RectGrad:
\begin{equation}
r_{l-1}^r=W_lM_l \left(r_{l}^r\odot\mathbb{I}(\textbf{A}_l \odot r_{l}^r>\tau) \right)
\end{equation}
can be formalized as a filtering rule for $F^r_l(W_l M_l)=W_l M_l diag(\mathbb{I}(\textbf{A}_l \odot r_{l}^r>\tau))$, and is negative filtering rule.
\end{proposition}
Note that LRP and DTD share the $z^+$-rule for the main propagation, we have:
\begin{proposition}
The backpropagation rule of $z^+$rule:
\begin{equation}
r_{(l-1)_{[i]}}^{+}=\sum_j{\left (\frac{z_{l_{[ij]}}^+}{\sum_k{z_{l_{[kj]}}^+}} \right )r_{l_{[j]}}^{+}}
\end{equation}
can be formalized as a filtering rule for $F^{+}_l(W_l M_l) =\gamma W_ldiag(\mathbb{I}(W_{l}>0)) M_l $ with bottom process $r_0\odot \textbf{x}$, where $\gamma$ is a normalization term to keep $\sum_i{r_{(l-1)_{[i]}}^{+}}=\sum_j{r_{l_{[j]}}^{+}}$, and is negative filtering rule.
\end{proposition}

The prooves for Propositions 1,2, and 3 are provided in Supplement A. The crux of the proof lies in observing that all modifications occur at the ReLU layer, ensuring that activations $\textbf{A}_{(l-1)}\ge 0$. Since these techniques eliminate the negative terms of $r_{l}$, the results should increase. Furthermore, for $z^+$ rule in Proposition 3, our proof follows a core idea similar to that of \cite{ancona2017towards}, which establishing the equivalence of the $z$-rule to Grad$\odot$Input. Importantly, $\gamma$ does not impact the results, as the attribution is consistently normalized to the range [0,1] before visualization, so the normalization term $\gamma$ is disregarded in our subsequent analysis, focusing solely on visualization performance.

The above backpropagation rules are not only essentially NFR but also exhibit notable similarities in their behaviors. In Table \ref{sample}, we show an intuitive visual comparison. It can be observed that the interpretation results of these methods exhibit a strong similarity when the bottom process is excluded. Specifically, for untrained models with random weights, the results of these methods effectively align with the input image. In fact, we can achieve such results:

\begin{figure}[t]
\vskip 0.2in
\centering

\begin{minipage}{0.4\linewidth}
\centering
\includegraphics[width=1\textwidth]{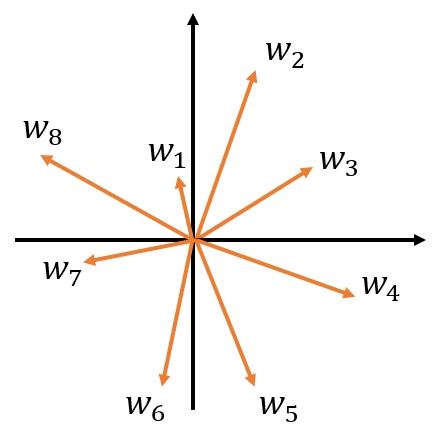}

(a)
\end{minipage}
\begin{minipage}{0.4\linewidth}
\centering
\includegraphics[width=1\textwidth]{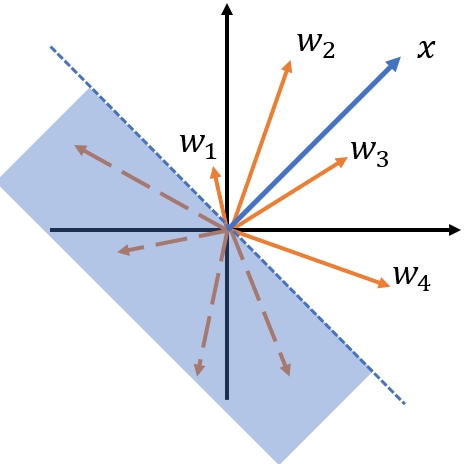}

(b)
\end{minipage}

\begin{minipage}{0.4\linewidth}
\centering
\includegraphics[width=1\textwidth]{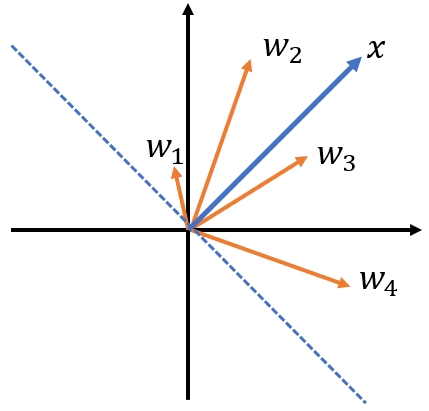}

(c)
\end{minipage}
\begin{minipage}{0.4\linewidth}
\centering
\includegraphics[width=1\textwidth]{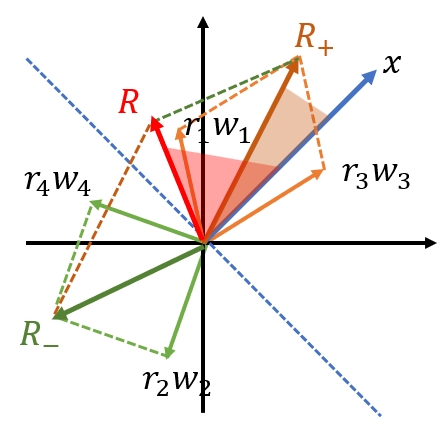}

(d)
\end{minipage}

\caption{ The geometric intuition of the proof of Theorem \ref{theorem1}. Masking negative vectors can improve input alignment. }
\label{illustration}

\vskip -0.2in
\end{figure}

\begin{theorem}
    \label{theorem1}
    In a random one hidden-layer neural network where every entry of weights is assumed to be independently isotropy distributed with a zero mean, if the number of filters $N$ is sufficiently large, the results of the NFR can be approximated as: 
    \begin{equation}
        R(\textbf{x})\approx \textbf{x}
    \end{equation}
    where $R(\textbf{x})$ and $\textbf{x}$ are all normalized to the same $L_2$-norm.
\end{theorem}
 The details of the proof are provided in Supplement B.

Our results are similar to those of Nie et al. \cite{nie2018theoretical}, albeit under milder assumptions. In their work, they required the assumption that the weights follow an i.i.d. Gaussian distribution with zero mean. It can be shown that under this assumption, the weight vector is also isotropic. Furthermore, their theory is restricted to GBP, a specific instance of NFR. As such, their conclusions represent a special case of our Theorem \ref{theorem1}.

The geometric intuition of the proof is illustrated in Fig.\ref{illustration}. Consider a one hidden-layer neural network with 8 hidden units. As shown in Fig.\ref{illustration}(a), the input is two-dimensional, and we represent the hidden units with weight vectors $\{w_1, w_2, ..., w_8\}$. Next, we input a vector $\textbf{x}$ in Fig.\ref{illustration}(b). It is important to note that the ReLU activation function masks negative activations. Specifically, if $\left < w_i, \textbf{x} \right > < 0$, it indicates that $w_i$ is on the opposite side of $\textbf{x}$, as depicted in the blue area of the picture.
Subsequently, we eliminate these masked vectors in Fig.\ref{illustration}(c) and leverage the backward information $r_i$ to obtain the final results as shown in Fig.\ref{illustration}(d). Since the backward information $r_i$ consists of both positive and negative terms, weight vectors on the same side as the input will appear to be on the opposite side due to the ReLU activation. We categorize the positive vectors as $R_+$ and the negative vectors as $R_-$.

The raw gradient is then calculated as $R = R_+ + R_-$. It becomes evident that the NFR method yields results where $R_+$ comprises vectors on the same side as the input, resulting in a better alignment with the input $\textbf{x}$ compared to the raw gradient. Specifically, in cases where the vectors are isotropic and sufficiently numerous, $R_+$ can be approximated as the input $\textbf{x}$.

In the multi-layer case, the key is the transitivity of alignment. Specifically, the NFR features from the previous layer enable the current layer’s output to align with the current activation. We can ensure that the entire backpropagation process forms a cascade of aligned inputs only if we demonstrate that aligning the current activation improves the alignment of subsequent layers.
In addition, there is another issue here that needs to be addressed: during the proof of Theorem \ref{theorem1}, we used the independence assumption, which transforms the expectation of the final result into the product of two expectations. However, due to the previous alignment process, the gradient and the activation are no longer independent at this point, so this transformation no longer holds, and we cannot directly use Theorem 1 to get the result of NFR to be aligned to the input. Fortunately, we further prove the following theorem:

\begin{theorem}
    \label{theorem2}
    In a one hidden-layer neural network where any two of all the weights of the nerons in the hidden layer have the same $L_2-Norm$ are orthogonal to each other. and the activation is $A$. Suppose $r$ is the attribution backpropagate to the hidden layer, then consider two middle layer attribution $r_a$ and $r_b$, where $\alpha(r_a,A)=1$ , and every entry in $r_b$ is assumed to be i.i.d for some unknown distribution, then the final results of NFR of $r_a$ and $r_b$ have:
    \begin{equation}
        \alpha(R_a,\textbf{x})\ge \alpha(R_b,\textbf{x})
    \label{target}
    \end{equation}
    where equality holds if and only if all the $\textbf{A}_i$ is equal.
\end{theorem}

The details of the proof are provided in Supplement C. The conditions of Theorem 2 require that all weights be pairwise orthogonal and of equal length. These conditions are readily satisfied as the dimensionality increases. Specifically, for two weight vectors drawn from an i.i.d. Gaussian distribution, the probability that they are approximately orthogonal and of equal length approaches 1 as the dimension $N$ becomes sufficiently large \cite{highdimension}.

According to this result, the alignment using intermediate layer weighting outperforms random weighting, so this lets the cascade alignment of the NFR hold: the alignment of the previous layer backpropagation makes the gradient of the intermediate layer become closer to the intermediate layer activation, and the weighting that tends to the intermediate layer activation in turn leads to a better alignment of the subsequent backpropagation results.

\subsection{Justifications of Perplexing Behavior}
\label{sec3.3}
Attribution methods often produce visually compelling explanations  but are insensitive to weight randomization, Theorems \ref{theorem1} and \ref{theorem2} provide justifications for these behaviors.

\textbf{Interpretable Visualizations.}
Theorem 1 demonstrates that attribution results align with input data, which is inherently human-understandable. This alignment is achieved via a weighted average of activation weights, ensuring that decision-relevant input information is emphasized while non-activation weights are excluded from explanations. For instance, Table \ref{sample} illustrates that in random model attribution highlights input features, while in pretrained model attribution focuses on edge information of the target objects. In addition, the final column of Table \ref{sample} shows that directly using activations for backpropagation (fully aligned as per Theorem 2) nearly reconstructs the original input with minimal distortion compared to other attribution methods. 

\textbf{Weight Randomization Insensitivity.}
Adebayo et al. \cite{adebayo2018sanity} observed that randomizing high-layer weights significantly alters model decisions while leaving attributions unchanged, raising concerns about explanation reliability. From our theory, this behavior is expected. Theorem 1 shows that even randomized weights can be aligned in a single hidden layer, while Theorem 2 indicates that alignment becomes more efficient with weight alignment. Randomizing high-layer weights does not disrupt high-layer alignment nor the subsequent backpropagation, resulting in minimal variation in attribution outcomes.

\textbf{Occlusion-type Reliability.}
Occlusion-type reliability tests involve removing the most relevant pixels and observing decision changes. Significant decision shifts confirm the importance of these pixels and the reliability of the attribution. According to our theory, NFR-based attributions reflect aggregated activation weights, meaning high-scoring pixels correspond to numerous activated neurons. Removing these pixels naturally disrupts the decision.

\begin{figure*}[t]
\vskip 0.2in
\centering
\begin{minipage}{0.3\linewidth}
\centering
\includegraphics[width=1\textwidth]{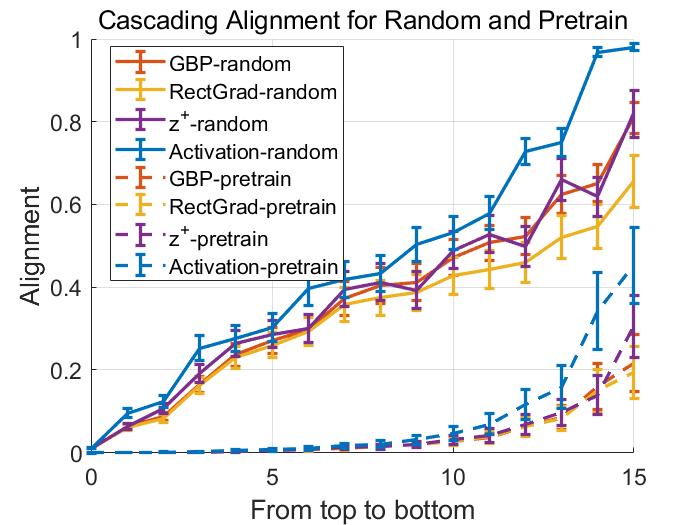}
(a)
\end{minipage}
\begin{minipage}{0.3\linewidth}
\centering

\includegraphics[width=1\textwidth]{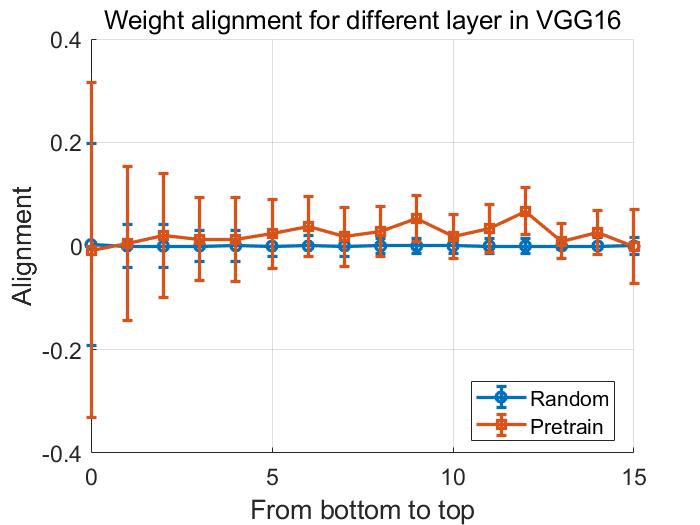}
(b)
\end{minipage}
\begin{minipage}{0.3\linewidth}
\centering
\includegraphics[width=1\textwidth]{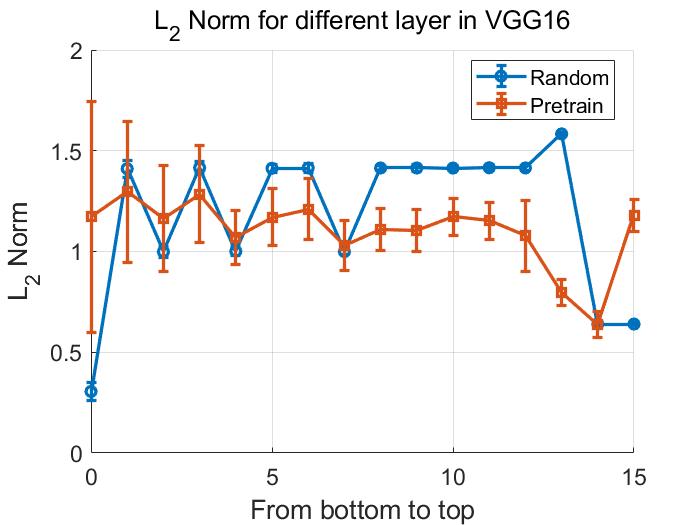}
(c)
\end{minipage}

\caption{(a) Cascadingly replacing the raw gradient with BP rules from GBP, RectGrad, and the $z^+$ rule, along with activations from the top layer to the bottom layer, results in a consistent increase in the cosine similarity between the final attribution and the input.
(b) For each layer, the cosine similarity between the weights of corresponding neurons is computed, where a similarity of 0 indicates orthogonality.
(c) For each layer, the $L_2$ norm of each neuron's weight is computed, with smaller error bars indicating lower standard deviation in the $L_2$ norm of the layer weights. }
\label{result_4-1}

\vskip -0.2in
\end{figure*}

\section{Experimental Validation}
\label{sec:experiment}
In this section, we present experiments to validate our theoretical framework. First, we assess whether the key conclusions and conditions derived from the theoretical analysis  remain valid when applied to practical models. Subsequently, we demonstrate that our theory provides better predictions of attribution method behaviors compared to existing theories.

\subsection{Verification of Claims and Conditions}
\label{sec:4.1}
Our theory emphasizes three aspects to be verified: (1) visual verification of input alignment, (2) quantitative validation of cascade alignment, and (3) evaluation of the Theorem \ref{theorem2} conditions.

\textbf{Settings.} Our experiments are conducted on the ImageNet validation set\cite{russakovsky2015imagenet} using the VGG16 model\cite{simonyan2014very}. Two configurations are evaluated: \emph{Random}, where the model weights are randomly initialized, and \emph{Pretrain}, which uses the pretrained weights from torchvision.models.

\textbf{Visual Inspection.} Table \ref{sample} illustrates the visual consistency of results across four attribution methods: GBP, RectGrad, LRP, and DTD. Despite differing implementations, these methods share the key backpropagation rule, NFR, ensuring consistent interpretability of their results. 
In the random model, isotropy, as outlined in Theorem 1, is satisfied. As a result, the explanations are effectively aligned to the input. Conversely, the pretrained model violates isotropy, leading to weaker input alignment and a loss of background details, such as leaves and branchs, while still effectively preserving crucial features of the target object, like the edges of the bird. This is because background information cannot activate the neurons in the higher layers, and it is lost in the NFR backpropagation.

In addition, the last columns (Activation) show that direct use middle-layer activations retain more input information. 
According to Theorem \ref{theorem2}, the essence of cascade alignment is that the alignment of the higher layers allows the gradient passed to the target layer to become more similar to activation. Therefore, if we use activations directly at the target layer, the information that disappears in the final attribution can only originate from the layers which below the target layer (from the input to the target layers).
This helps pinpoint the input information loss in different layers. According to the Table \ref{sample}, color information is lost at the bottom layers, causing the final results to focus on edges and textures even when activations are directly weighted. Conversely, background information is lost at higher layers, allowing branches to appear in Activation but not in other attribution results. This distinction underscores the model's progressive filtering of input informations across layers.

\textbf{Cascade Alignment.}
Combining Theorem \ref{theorem1} and Theorem \ref{theorem2}, a key corollary of our theory is that cascade alignment, i.e., more layers choose NFR allows the final result to have better input alignment. To quantify cascade alignment, we propose cascade substitution:
\begin{Definition}
Suppose the raw rule of backpropagation is:
$
    r_{l-1}=W_l M_l r_{l}
$
and the target rule $F$ is:
$
    r_{l-1}=F(r_{l})
$
N-level Cascade Substitution of the target rule is $r_0^N$ which:
\begin{equation}
    r_{l-1}^N=\left\{
             \begin{array}{lr}
             F(r_{l}), &  L-N<l\le L\\
             W_l M_l r_{l}, & 0<l\le L-N \\
             \end{array}
\right.
\end{equation}
\end{Definition}
The input alignment of GBP, RectGrad, $z^+$ rule and the direct use of activation weighting (as stated in Theorem 2) are presented in the Fig.\ref{result_4-1} (a) for both the randomly weighted and pretrained models. As the number of substitution rules increases, the alignment effect shows a clear upward trend across all NFRs and Activation. In addition, the alignment performance of the randomly weighted model is significantly better than that of the pretrained model, consistent with Theorem \ref{theorem1}.  Notably, the direct use of activation consistently outperforms all NFR methods in alignment, indicating that fully aligned intermediate layers surpass standard NFRs, consistent with Theorem \ref{theorem2}.

\textbf{Orthogonality and Norm Equality.}
Theorem \ref{theorem2} holds when weights are pairwise orthogonal and have equal norm lengths. Fig.\ref{result_4-1}(b) and (c) depict the cosine similarity between neuron weights and the norm lengths for both the randomly weighted and pretrained models. For the random model, the neuron weights are nearly orthogonal (cosine similarity $\approx$ 0), with minimal variation in norm lengths, satisfying the conditions of Theorem \ref{theorem2}. In contrast, the pretrained model exhibits higher cosine similarity and greater variability in norm lengths, which partially explains its inferior alignment performance compared to the randomly weighted model.

\begin{figure*}[t]
\vskip 0.2in
\centering
\begin{minipage}{0.11\linewidth}
\centering
\begin{minipage}{1\linewidth}
\centering
Gaussian
\includegraphics[width=1\textwidth]{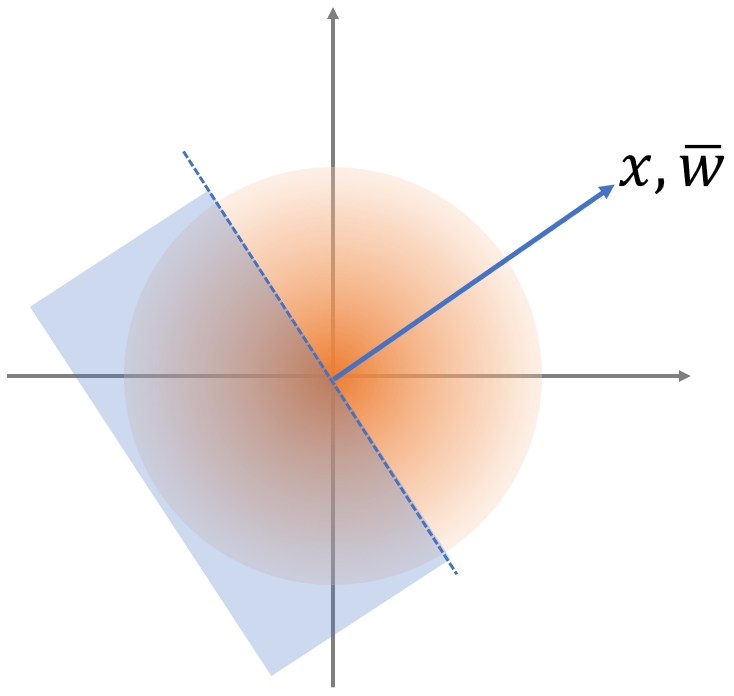}
\end{minipage}

\begin{minipage}{1\linewidth}
\centering
Ring
\includegraphics[width=1\textwidth]{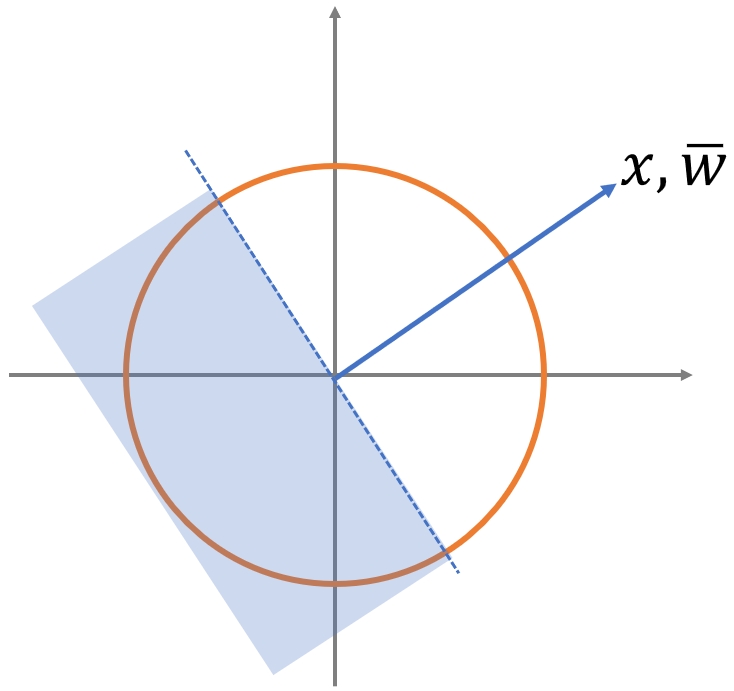}
\end{minipage}

\begin{minipage}{1\linewidth}
\centering
Uniform
\includegraphics[width=1\textwidth]{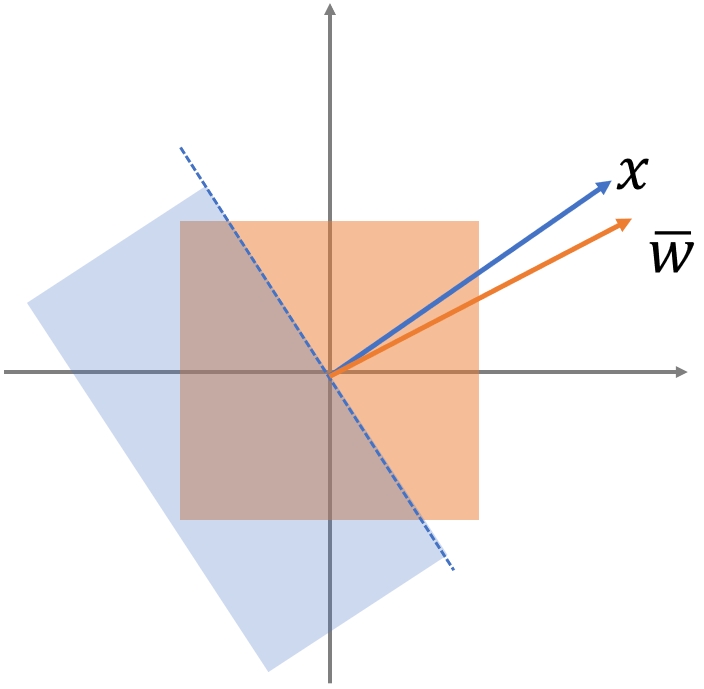}
\end{minipage}
\end{minipage}
\begin{minipage}{0.8\linewidth}

\begin{minipage}{0.12\textwidth}
\centering
Image
\includegraphics[width=1\textwidth]{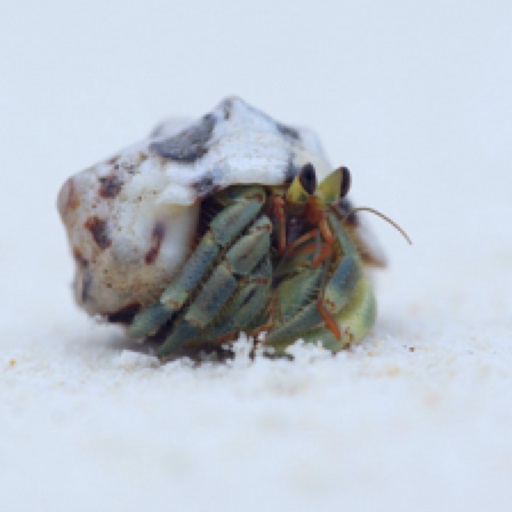}
\end{minipage}
\begin{minipage}{0.82\textwidth}
\begin{minipage}{1\textwidth}
\centering
Gaussian
\includegraphics[width=1\textwidth]{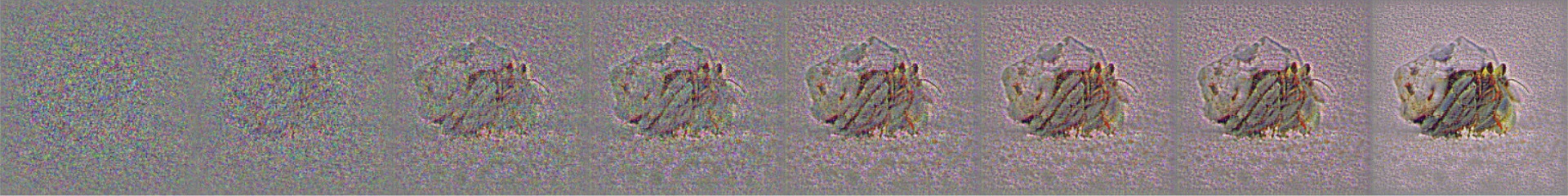}
\end{minipage}

\begin{minipage}{1\textwidth}
\centering
Ring
\includegraphics[width=1\textwidth]{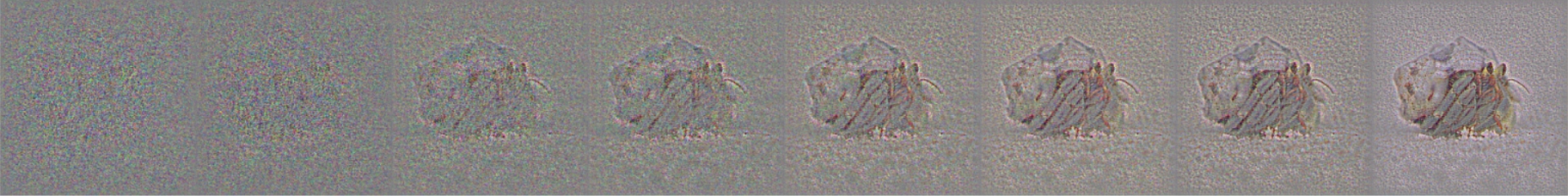}
\end{minipage}

\begin{minipage}{1\textwidth}
\centering
Uniform
\includegraphics[width=1\textwidth]{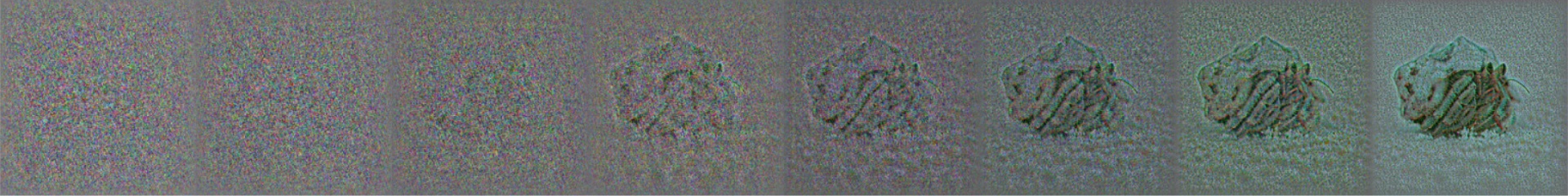}
\end{minipage}

\end{minipage}
\end{minipage}
\caption{\textbf{Left:} Illustration of weight randomization distributions: the top shows a Gaussian distribution, the middle a non-Gaussian but isotropic distribution, and the bottom a non-isotropic distribution. Our theory predicts that the bottom distribution deviates from the input $\textbf{x}$. \textbf{Right:} All randomization methods show alignment with the input image. However, the lack of isotropy in the weight distribution causes a more significant distortion in the color results of uniform randomization compared to Gaussian or ring randomization.}
\label{uniform}

\vskip -0.2in
\end{figure*}

\begin{figure*}[t]
\vskip 0.2in
\centering
\begin{minipage}{0.3\linewidth}
\centering
GBP
\includegraphics[width=1\textwidth]{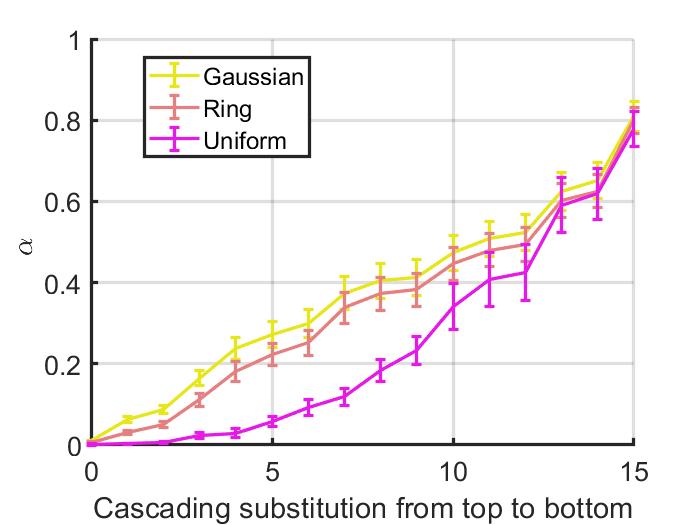}
\end{minipage}
\begin{minipage}{0.3\linewidth}
\centering
RectGrad
\includegraphics[width=1\textwidth]{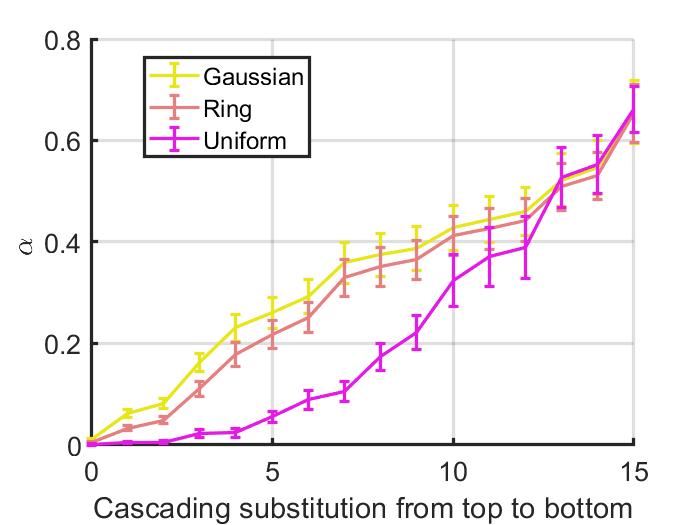}
\end{minipage}
\begin{minipage}{0.3\linewidth}
\centering
$Z^+$
\includegraphics[width=1\textwidth]{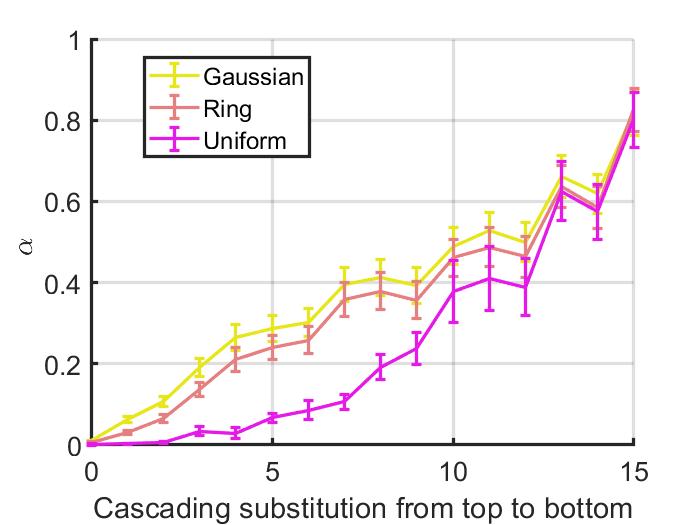}
\end{minipage}

\caption{The quantitative results of cascade substitution for different distributions show that ring randomization provides similar alignment performance to Gaussian randomization, while uniform randomization performs significantly worse across all NFR methods. }
\label{result_uniform}

\vskip -0.2in
\end{figure*}

\subsection{Experimental Comparison with Existing Theory}
\label{sec:empirical}
Several theoretical studies have attempted to explain the perplexing behaviors analyzed in this paper. In this section, we present three experiments demonstrating that our theory predicts attribution behaviors more effectively than existing theories.

\textbf{GBP as partial input recovery.}
Existing theoretical frameworks \cite{nie2018theoretical} focus primarily on input reconstruction to justify interpretable patterns, but they can only justify the behavior of GBP. Our theory has similar results but rely on milder assumptions, and can extend to more generalized modified BP attributions, like RectGrad and LRP. Moreover, unlike they require the model weights distributed according to a Gaussian distribution, our theory is based on isotropic weight assumptions. This enables us to address cases involving non-Gaussian isotropic randomizations and non-isotropic randomizations. Specifically, their theory posits that the performance is superior when the weights follow a Gaussian distribution, regardless of isotropy. In contrast, our theory suggests that the alignment performance with isotropic weight distributions is comparable to that of Gaussian distributions, and significantly superior to nonisotropic distributions.

In our experiment, we use a uniform distribution on a ring (see Fig. \ref{uniform}, middle left) as the non-Gaussian isotropic distribution. In addition, we choose uniformly distributes weights within $[-1, 1]$ for each entry (Fig. \ref{uniform}, bottom left) as the non-Gaussian distribution. Prior theory \cite{nie2018theoretical} predict that Gaussian alignment should outperform both cases, but our theory suggests that Gaussian and ring distributions yield comparable alignment, both superior to the uniform distribution.

Using the ImageNet validation set and the VGG16 model, we evaluated alignment under the above three randomizations. Employing cascade substitution, we analyze the impact from the changes of weight distribution. Fig. \ref{uniform} illustrates qualitative results, while Fig. \ref{result_uniform} provides quantitative evidence. The findings confirm our predictions: ring randomization aligns similarly to Gaussian, whereas uniform randomization performs significantly worse across all NFR methods. Additionally, the lack of isotropy in uniform randomization distorts attribution results, as shown in Fig. \ref{uniform}. This phenomenon, unaccounted for by Gaussian-based theories, is effectively explained by our isotropy-based framework.

\textbf{ Nonnegative matrices multiplication converges to rank 1.}
Sixt et al. \cite{sixt2020explanations} show that the insensitivity derive from a special property: sufficient number of nonnegative matrices multiplication will converges to rank 1. In other words, the vectors becomes the same direction after backpropagation. Unlike prior theories requiring nonnegative matrix multiplication \cite{sixt2020explanations}, our framework is grounded in NFR. This distinction enables us to account for the behaviors of GBP and RectGrad, which do not adhere to nonnegative matrix multiplication constraints but are NFR. Furthermore, the existing theory asserts that attribution results remain invariant to all changes in the top-layer weights. In contrast, our theory highlights the critical role of randomizations: weight randomizations do not cascade alignment process. If the change destroy the cascade alignment, the final attributions will also change.

In experiments, we implement a weight removal strategy rather than randomization in sanity checks\cite{adebayo2018sanity}. In Fig.\ref{parameter}, random denotes weights replaced by Gaussian-distributed random vectors, while removal involves setting a fraction of weights to zero. For example, in the linear layer, only 0.25$\%$ of the neurons from the input dimension are retained, and in the convolutional layers, only the first row and column of kernel weights are preserved. As shown in Fig.\ref{parameter}, the removal operation leads to a marked loss of alignment capability than randomizations. This suggests that randomization which does not disrupt cascade alignment is indeed an important factor in explaining why the results remain unchanged.

\begin{figure*}[t]
\vskip 0.2in
\begin{minipage}{0.3\linewidth}
\centering
\includegraphics[width=1\textwidth]{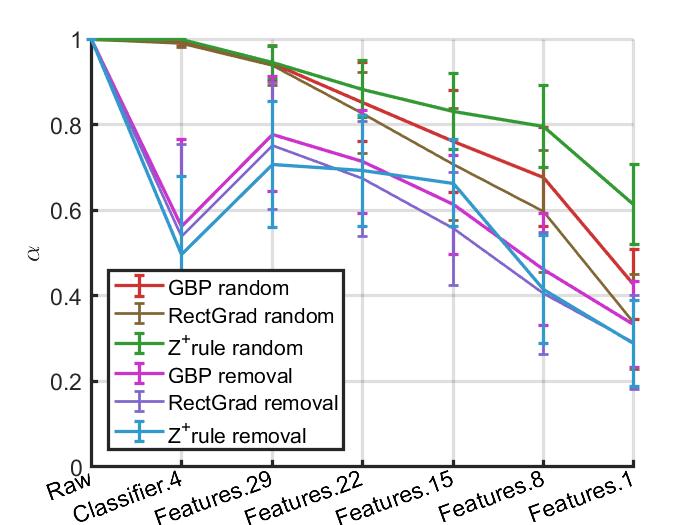}
\end{minipage}
\begin{minipage}{0.66\linewidth}

\begin{minipage}{0.15\textwidth}
\centering
Image

\includegraphics[width=1\textwidth]{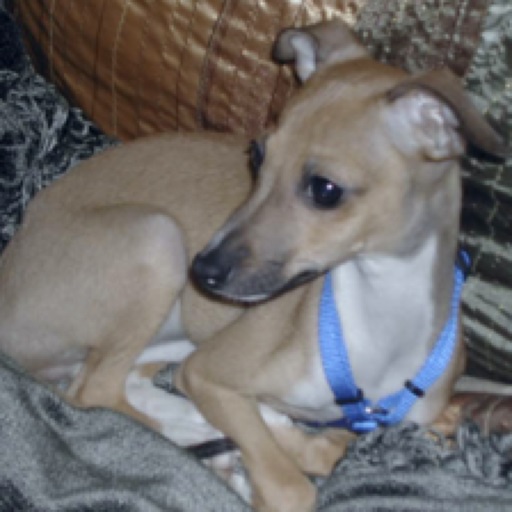}
\end{minipage}
\begin{minipage}{0.82\textwidth}
\begin{minipage}{1\textwidth}
\centering
Random

\includegraphics[width=1\textwidth]{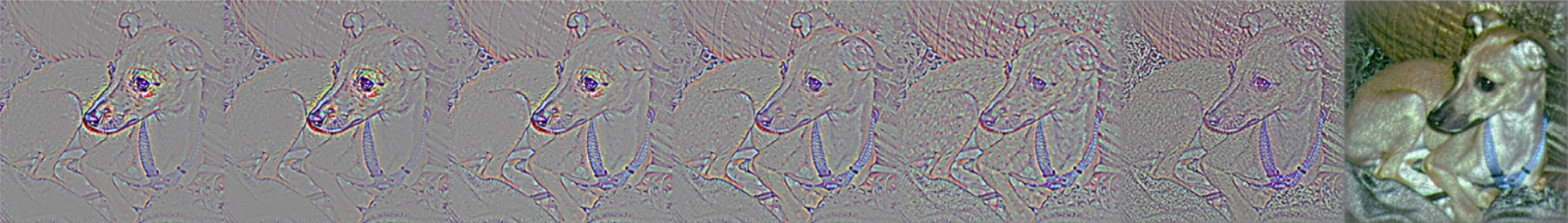}
\end{minipage}

\begin{minipage}{1\textwidth}
\centering
Removal

\includegraphics[width=1\textwidth]{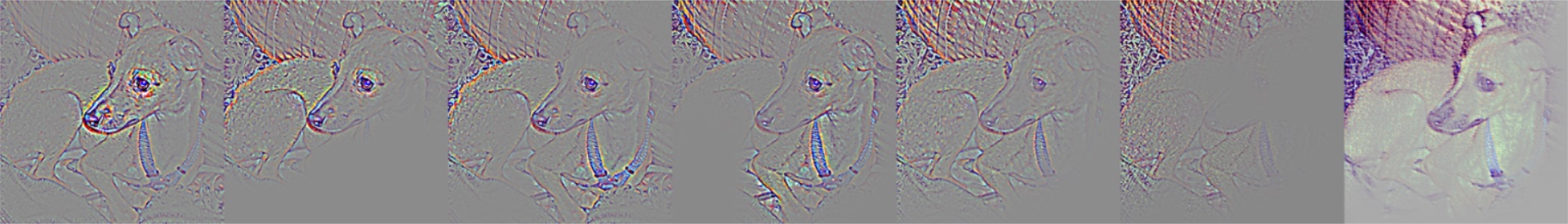}
\end{minipage}

\end{minipage}
\end{minipage}
\caption{The removal operation results in a significant loss of content in the interpreted results, suggesting that these methods are not inherently insensitive to weight changes in the top layers. }
\label{parameter}

\vskip -0.2in
\end{figure*}

\textbf{Large activations determine decision-making process.}
Binder et al.\cite{shortcoming} observed that many modified BP attribution methods are insensitive to weight changes while performing well in occlusion-type reliability evaluations. They attributed this to model decisions relying primarily on large activation values, which remain unaffected by randomization. In contrast, our theory suggests that insensitivity to weight changes is independent of activation magnitude. To test this hypothesis, we designed experiments to evaluate whether activation magnitude is a key factor in this insensitivity.

Using the pretrained VGG16 model from torchvision, we compute alignment between input and attributions (GBP, RectGrad, and $z^+$ rule) for the top 10$\%$ of activations (large activations, denoted as \emph{max}) and the bottom 10$\%$ of non-zero activations (small activations, denoted as \emph{min}) backpropagating from each layer to the input. As shown in Fig. \ref{result_minmax}, a notable difference between the explanations of large and small activations shows in the bottom layers. However, this difference diminishes to zero in higher layers. This indicates that large activation is not the determine factor. Even when large activations are excluded and only small activations are considered, the explanation results remain unchange.

\begin{figure}[t]
\vskip 0.2in
\centering
\begin{minipage}{0.9\linewidth}
\centering
\includegraphics[width=1\textwidth, height=0.67\textwidth]{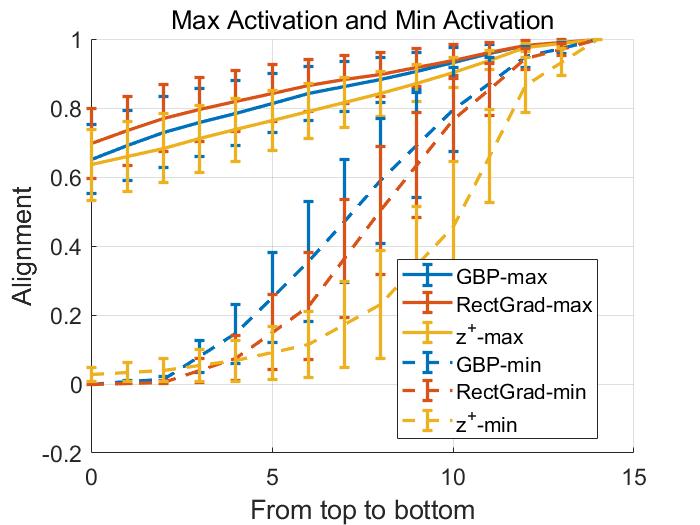}
\end{minipage}

\caption{As the target layer shifts from the bottom to the top, the attributions for both large and small activations converge to the original attributions.}
\label{result_minmax}

\vskip -0.2in
\end{figure}

\begin{figure*}[t]
\vskip 0.2in

\centering
\begin{minipage}{0.49\textwidth}

\centering
\includegraphics[width=0.9\textwidth]{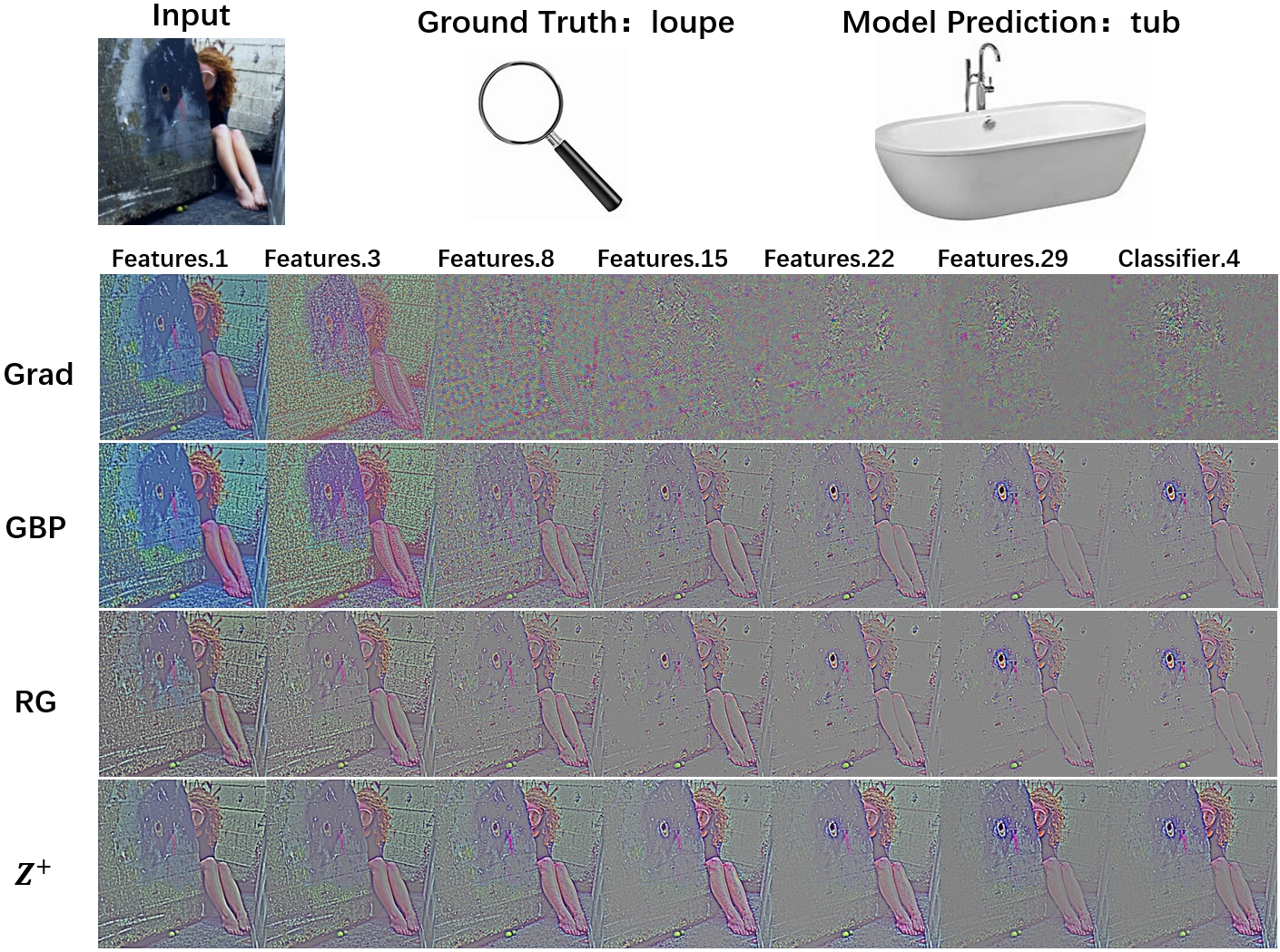}
\end{minipage}
\begin{minipage}{0.49\textwidth}
\centering
\includegraphics[width=0.9\textwidth]{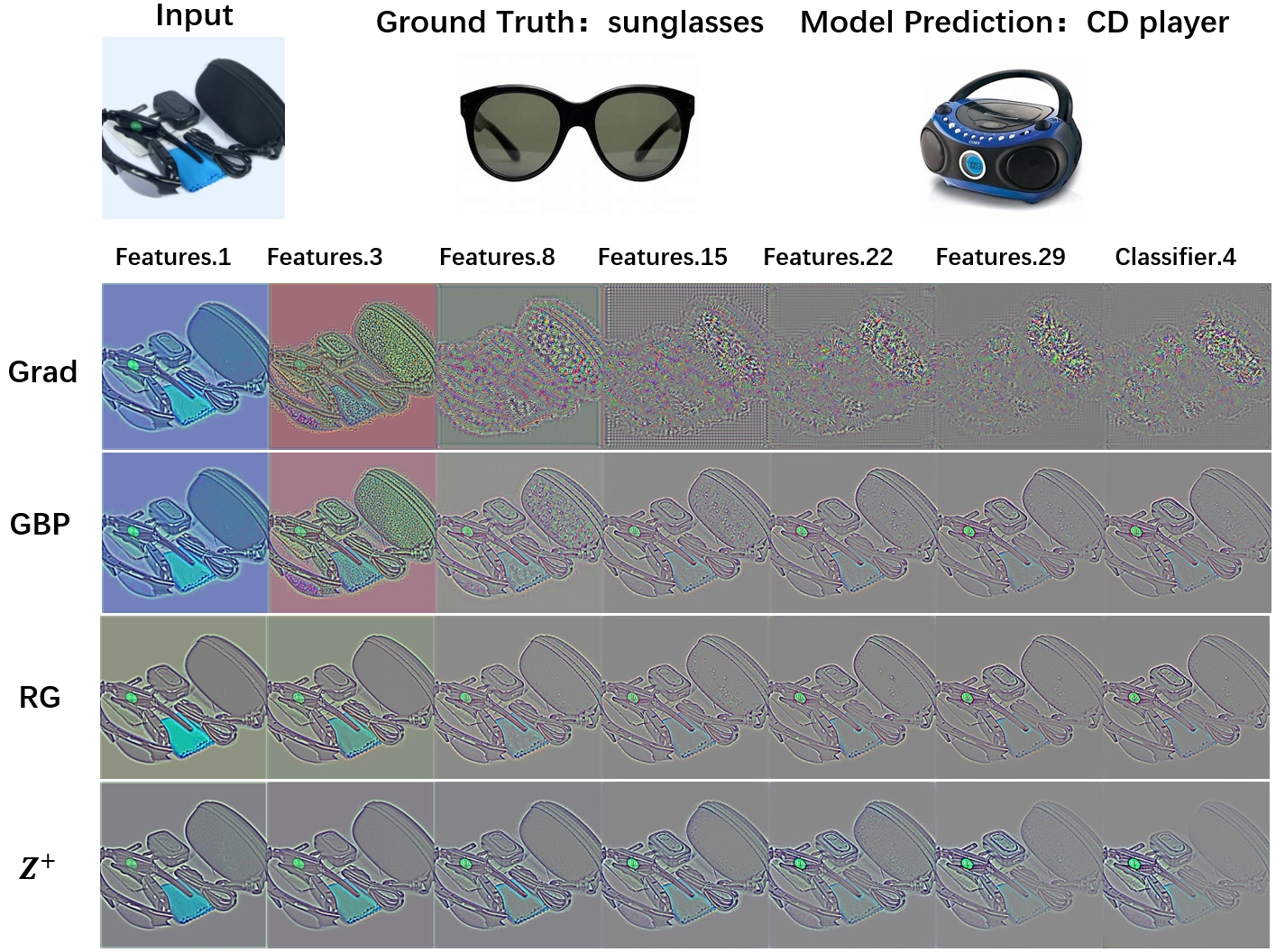}
\end{minipage}

\caption{\textbf{Left}: The correct category is loupe (633), misclassified as tub (876). The figure clearly shows that the model highlights the drain as tub-related content, indicating a misclassification error due to a representation flaw. \textbf{Right}: Sunglasses (836) are misclassified as a CD player (485). The misclassification occurs because the model focuses on the legs of the glasses with buttons, which deviate from typical sunglasses.}
\label{blindpoint}

\vskip -0.2in
\end{figure*}
\section{Application of Input Alignment}
\label{sec:application}
The primary insight from our theory is that modified BP attributions do not assign scores to inputs but instead reflect the information used in model prediction. This perspective allows a more granular and precise understanding of the decision-making process. Specifically, we propose two straightforward approaches: (1) identifying the layer where failure cases originate by analyzing the effect of successive layers on input information utilization, and (2) reintroducing attribution results into the model to determine whether the information they provide sufficiently supports the decision.

\subsection{Cascade Comparisons for Analyzing Failure Cases}
In Section \ref{sec:4.1}, we demonstrate that Theorem \ref{theorem2} divides the NFR process of any intermediate layer into two components. The first component involves backpropagation from the output to the target layer, aligning the activations of the target layer. The second component propagates the alignment results from the target layer back to the input. Theorem \ref{theorem2} shows that direct use of target layer activation has better alignment ability. Therefore, by directly using activation weights instead of the first component's output, the final result isolates the utilized information of the second component without destroying the cascade alignment. Furthermore, by increasing the number of layers in the second component and comparing the input information before and after this change, we can quantify how additional layers affect the utilization of input information. 

Fig. \ref{blindpoint} demonstrates that tracking the evolution of input information across layers provides deeper insights into decision errors compared to conventional attribution techniques. For example, in the left case, the correct category loupe (633) is misclassified as tub (876). Unlike conventional attribution, which identifies relevant input for the entire model, our method shows how input information evolves through the model. Specifically, color and background information gradually fade, while features like the drain and human figure progressively dominate the model’s attention. This step-by-step analysis clarifies how tub-related features mislead the model into making an incorrect decision.

Similarly, in the right case, sunglasses (836) are misclassified as a CD player (485). Our method reveals that, while key features of the sunglasses remain prominent across layers, the misclassification occurs because the model erroneously emphasizes features atypical of sunglasses, such as the legs of the glasses and the presence of a prominent button. This nuanced analysis, enabled by our method, provides actionable insights into both feature relevance and the evolving importance of specific details leading to misclassification.

\subsection{Key Information Sufficiency}
Our theory posits that modified BP attribution primarily reveals information about the inputs utilized in decision-making, rather than merely providing input importance scores. While there is often substantial overlap between the two, such as the target object driving the decision being highlighted and background information being suppressed, our theory enables novel applications of attribution from an input-information perspective. For instance, the attribution results can be reintroduced into the model to observe changes in its decision-making process. This approach offers researchers an intuitive means to assess the model’s decision-making capabilities based on this specific information and to better understand the influence of contextual factors.

Let $A_L$  denote the input features of the last fully connected layer and $v_k$  the weight associated with decision $k$ in the same layer. We propose using a Manhattan distance weighted by $v_k$  to quantify how the model's decision shifts when critical input information identified by modified BP attribution is contrasted with the original input. We use $*$ represent the input is normalized attributions $\tilde{R}$:
\begin{equation}
    S=1-\frac{\|v_k^*A_L^*-v_kA_L \|_1}{\|v_kA_L\|_1}
\end{equation}
where  $\tilde{R}=\frac{(R-min(R))(max(\textbf{x})-min(\textbf{x}))}{max(R)-min(R)}+min(\textbf{x})$.
We term this metric \emph{Key Information Sufficiency (KIS)}. There has been significant work on assessing attribution reliability by masking low-scoring inputs to evaluate changes in decision-making\cite{samek2016evaluating,ancona2017towards}, e.g., Insertion metric insert the most relevant pixels to zero input and compare the output changes. However, these metrics differs fundamentally from direct input methods. Masking introduces artifacts \cite{artifact}, and the results are highly sensitive to masking threshold selection. Moreover, the Insertion metric erroneously removes negative scores, which can seriously affect the validity of the results. According to our theory, negative scores essentially represent the same critical input information in Modifed BP attribution, but from the perspective of traditional relevance scoring, negative scores are information that should be removed that does not support the decision.

 All experiments are conducted on the CIFAR-10 dataset using a ResNet18 model. Modified BP attribution is performed using GBP with straightforward NFR.

\begin{figure*}[t]
\vskip 0.2in
\centering
\begin{minipage}{0.3\linewidth}
\centering
KIS

$\tilde{R}$

\includegraphics[width=1\textwidth]{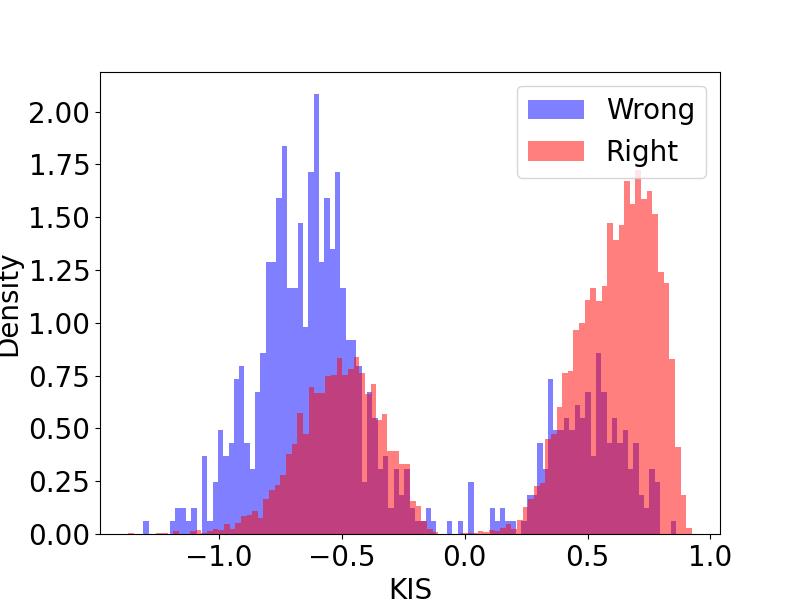}
(a)
\end{minipage}
\begin{minipage}{0.3\linewidth}
\centering
Insertion

$\mathbb{I}(R>\overline{R})\odot x$
\includegraphics[width=1\textwidth]{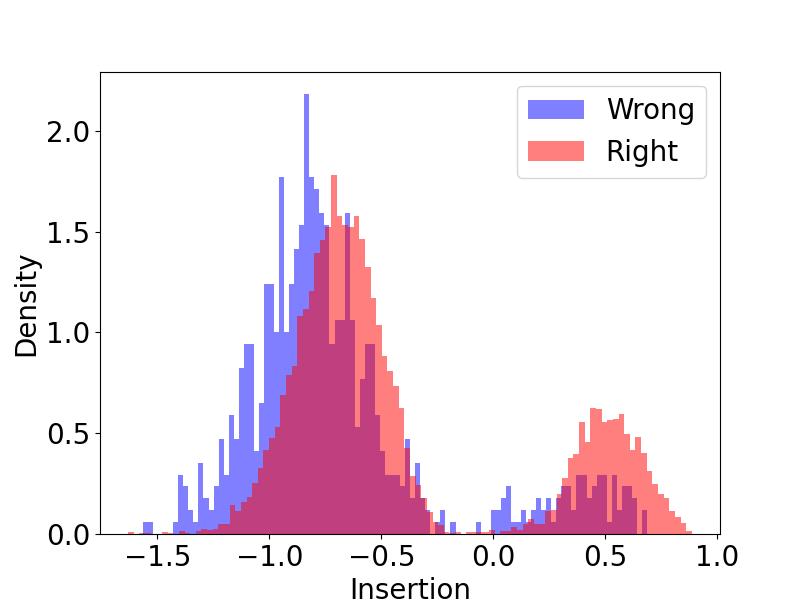}
(b)
\end{minipage}
\begin{minipage}{0.3\linewidth}
\centering
Insertion(ABS)

$\mathbb{I}(|R|>\overline{|R|})\odot x$
\includegraphics[width=1\textwidth]{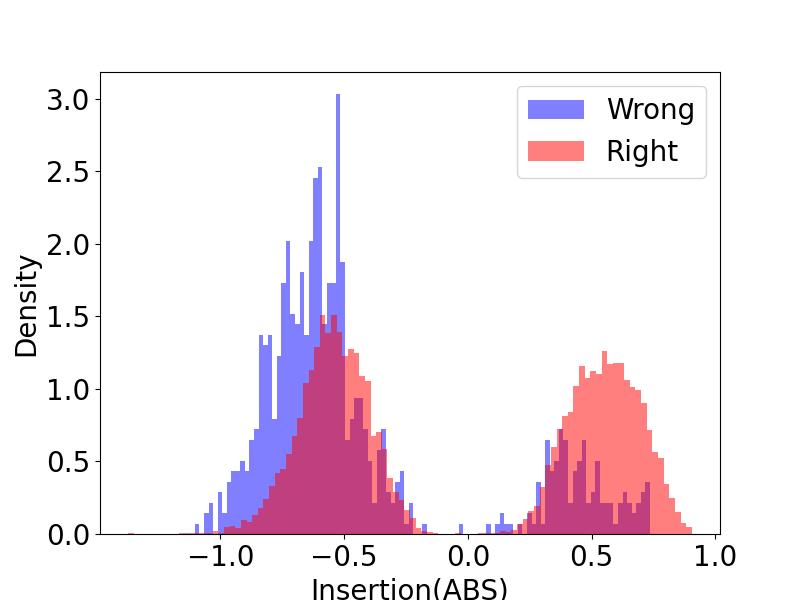}
(c)
\end{minipage}

\caption{(a) Distribution difference of KIS metrics between correct and incorrect samples on the ImageNet Validation Set for the ResNet18 model.
(b) Replacing the inputs in the KIS metrics from normalized attribution $\tilde{R}$ to Insertion metrics removes pixels with scores below the mean in input $x$.
(c) Similar to the Insertion metric, but only considers the magnitude of the score, not its sign.
 }
\label{result_bimodal}

\vskip -0.2in
\end{figure*}

\begin{figure*}[t]
\vskip 0.2in
\centering
\begin{minipage}{0.23\linewidth}
\centering
BadNet

\begin{minipage}{0.48\linewidth}
\centering
clean
\includegraphics[width=1\textwidth]{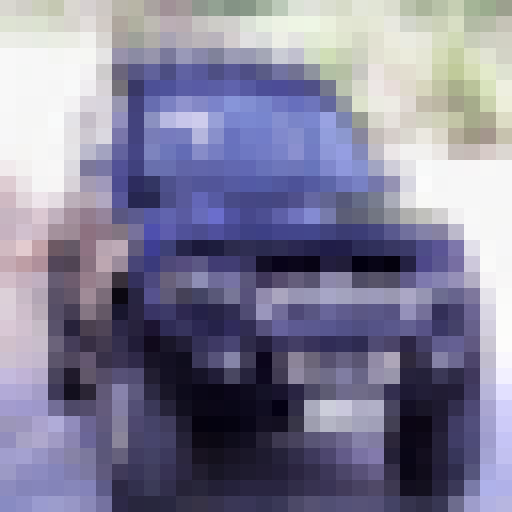}
\includegraphics[width=1\textwidth]{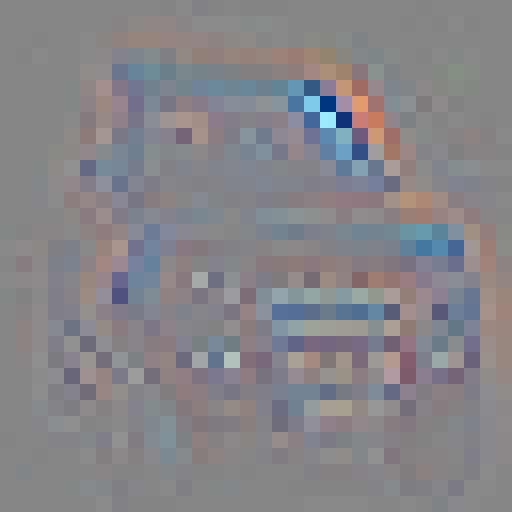}
\end{minipage}
\begin{minipage}{0.48\linewidth}
\centering
poison
\includegraphics[width=1\textwidth]{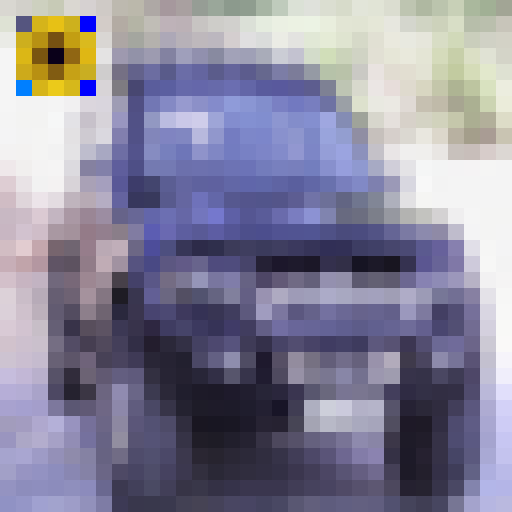}
\includegraphics[width=1\textwidth]{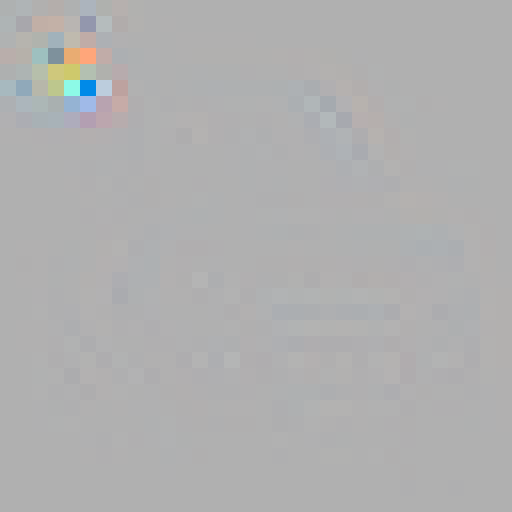}
\end{minipage}

\includegraphics[width=1\textwidth]{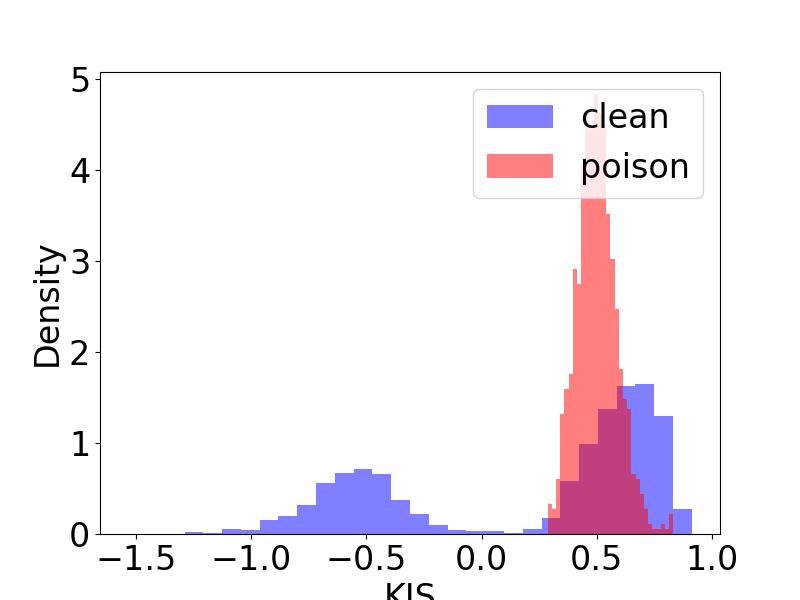}
\end{minipage}
\begin{minipage}{0.23\linewidth}
\centering
Dynamic

\begin{minipage}{0.48\linewidth}
\centering
clean
\includegraphics[width=1\textwidth]{images/raw_backdoor.png}
\includegraphics[width=1\textwidth]{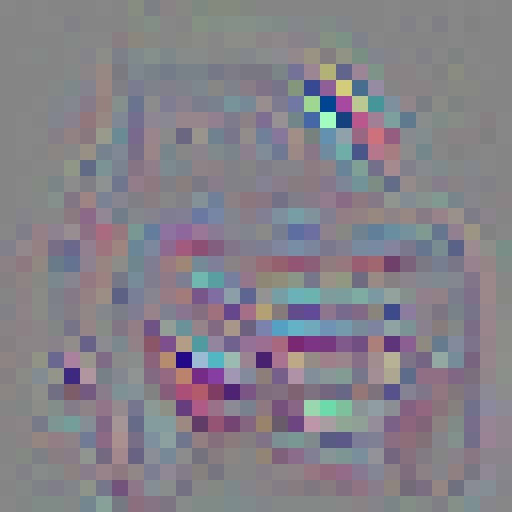}
\end{minipage}
\begin{minipage}{0.48\linewidth}
\centering
poison
\includegraphics[width=1\textwidth]{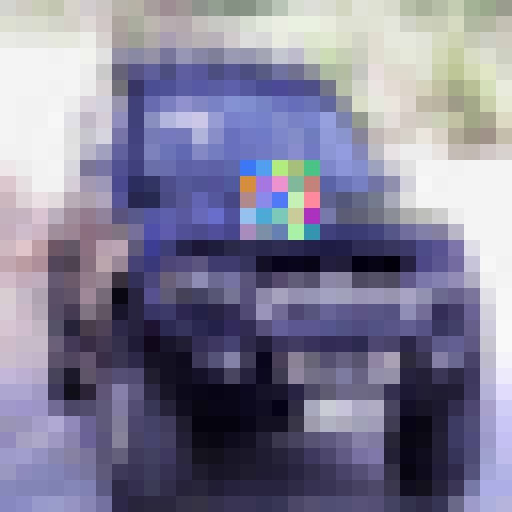}
\includegraphics[width=1\textwidth]{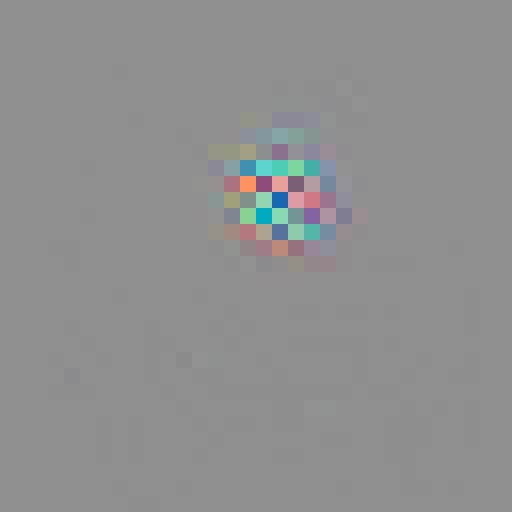}
\end{minipage}

\includegraphics[width=1\textwidth]{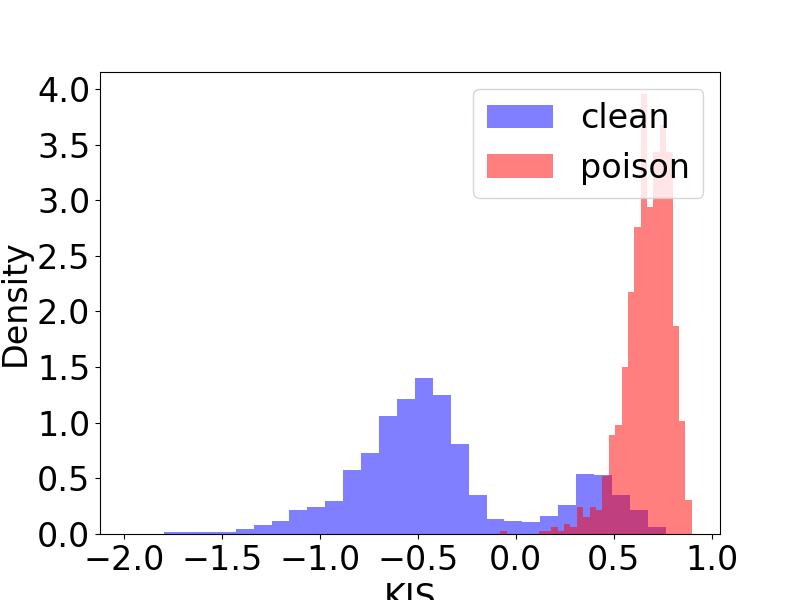}
\end{minipage}
\begin{minipage}{0.23\linewidth}
\centering
InputAware

\begin{minipage}{0.48\linewidth}
\centering
clean
\includegraphics[width=1\textwidth]{images/raw_backdoor.png}
\includegraphics[width=1\textwidth]{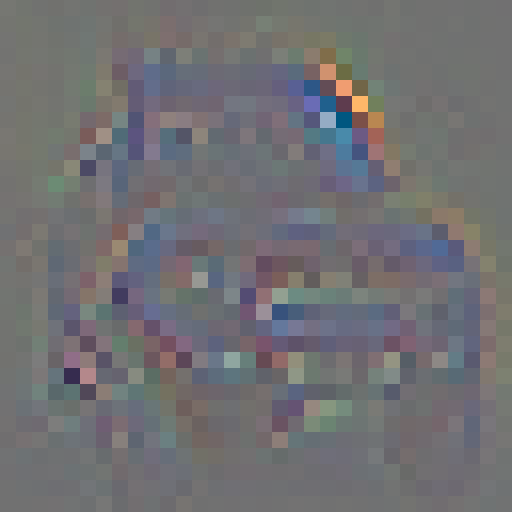}
\end{minipage}
\begin{minipage}{0.48\linewidth}
\centering
poison
\includegraphics[width=1\textwidth]{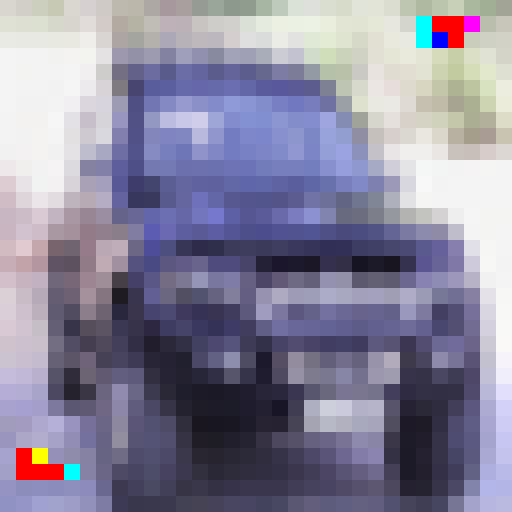}
\includegraphics[width=1\textwidth]{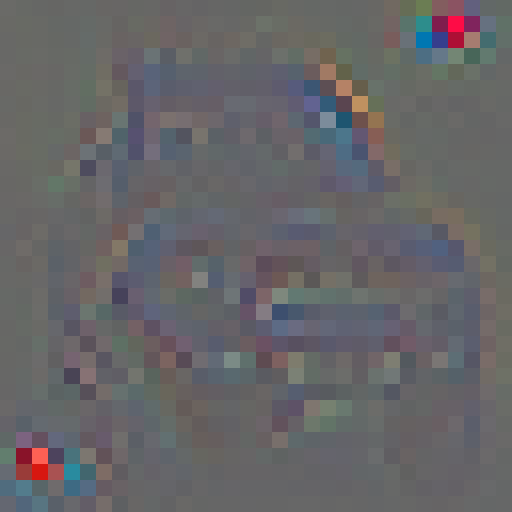}
\end{minipage}

\includegraphics[width=1\textwidth]{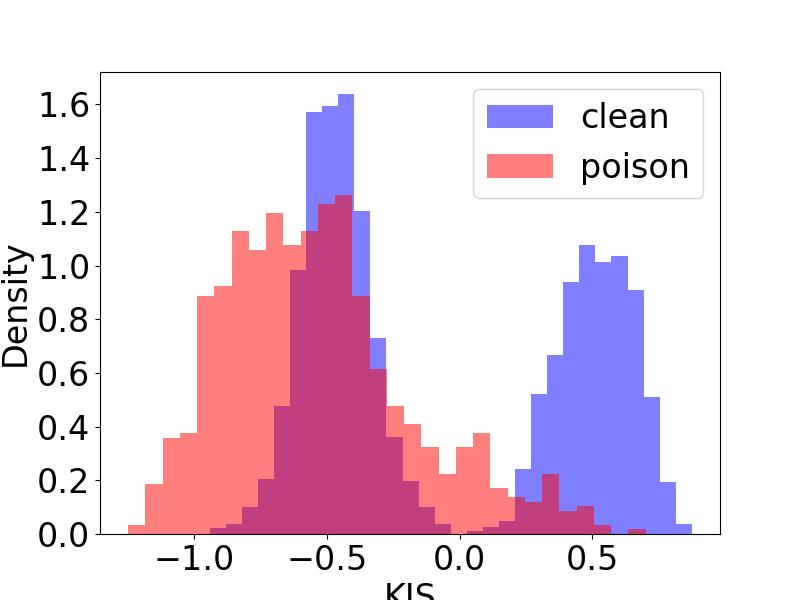}
\end{minipage}
\begin{minipage}{0.23\linewidth}
\centering
WaNet

\begin{minipage}{0.48\linewidth}
\centering
clean
\includegraphics[width=1\textwidth]{images/raw_backdoor.png}
\includegraphics[width=1\textwidth]{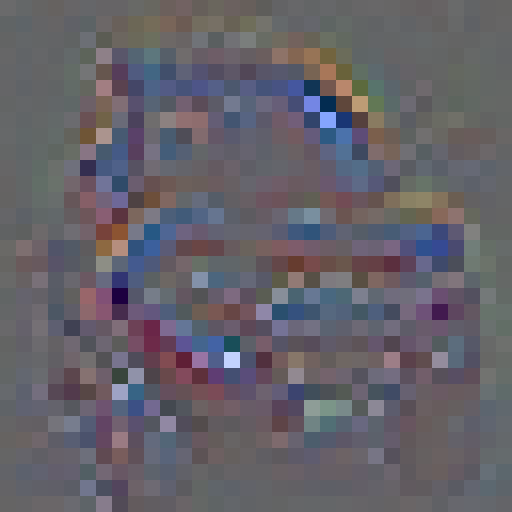}
\end{minipage}
\begin{minipage}{0.48\linewidth}
\centering
poison
\includegraphics[width=1\textwidth]{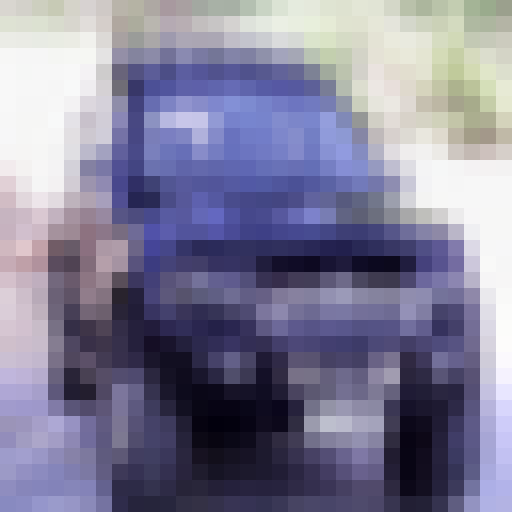}
\includegraphics[width=1\textwidth]{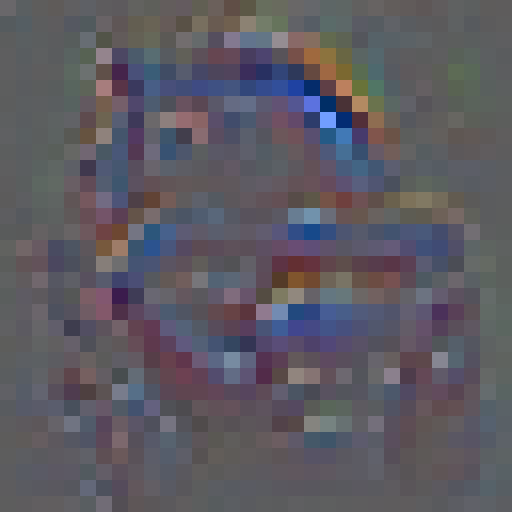}
\end{minipage}

\includegraphics[width=1\textwidth]{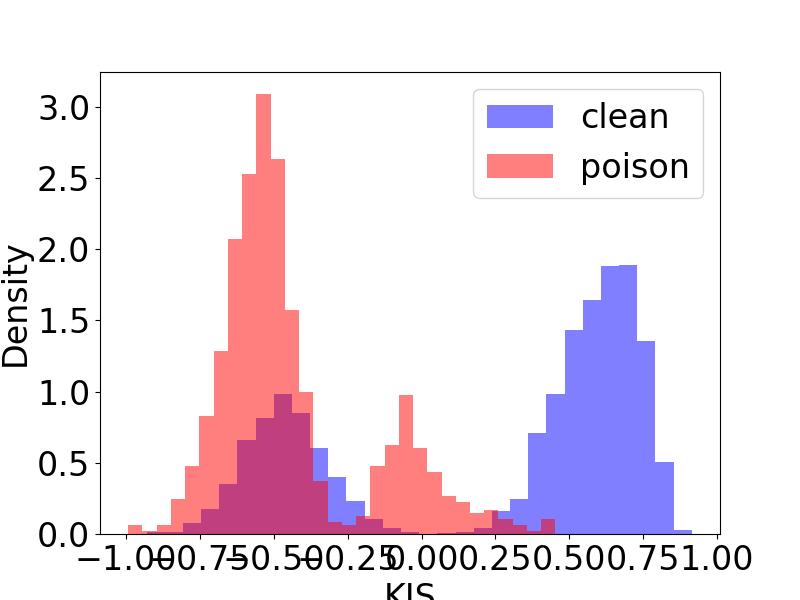}
\end{minipage}

\caption{KIS results for four backdoor attacks: BadNet \cite{badnet}, Dynamic \cite{dynamic}, InputAware \cite{inputaware}, and WaNet \cite{wanet}. For the toxic sample, the attributional explanations correctly highlight the backdoor. The KIS results for different attacks correspond to distinct peaks, with BadNet and Dynamic aligning with the right peak, and InputAware and WaNet aligning with the left.}
\label{backdoor}

\vskip -0.2in
\end{figure*}
\subsubsection{Performance on clean datasets}
\label{sec:5.1}
Fig. \ref{result_bimodal} shows the performance of the KIS metric on a model trained with 92.47$\%$ accuracy. The results display a clear bimodal distribution, with one peak above zero and another below, indicating two distinct decision-making types. The first type shows decisions that remain consistent with the original inputs when key information is used, while the second shows decisions that change significantly when only key information is considered.

In Fig. \ref{result_bimodal}(a), a strong correlation between these types and the generalizability is evident. When the key information from the attribution explanation is sufficient (right peak), the model's decisions exhibit better generalizability. In contrast, when key information is insufficient (left peak), the model is more prone to errors. Fig. \ref{result_bimodal}(b) shows the results of the insertion metric. To ensure consistency in the comparison, we adopted the same settings as in KIS when observing the decision changes. In this case, the correlation with generalizability is weak.

In Fig. \ref{result_bimodal}(c), we modify the approach by removing inputs with low absolute value scores rather than low-scoring inputs. It can be seen that the correlation with generalizability improves than traditional Insertion settings, because  negative score pixels are correctly preserved. However, it remains weaker than when the attribution results are directly fed into the model. This is because the occlusion operation introduces new artifacts that affect the decision-making process, whereas direct input of key information does not have this problem.

\subsubsection{Performance on Backdoor Attack}
Backdoor-based reliability evaluation is a common strategy for assessing attributional explanations. By injecting a backdoor into the model, researchers can establish a clear ground truth for these explanations: the decision is altered by the backdoor, and thus the attribution should correctly correspond to the backdoor. However, while backdoor attack research is advancing rapidly with many new designs, existing attribution evaluation methods primarily rely on BadNet and its variants \cite{backdooreval, backdooreval2}. Moreover, the relationship between backdoor features and normal features remains unclear, and it is uncertain whether these backdoors fully capture the behavior of features used in regular decision-making. In this section, we revisit the role of backdoor features in decision-making using KIS metrics.

\textbf{Settings.} We evaluated four common backdoor attack methods: BadNet \cite{badnet}, Dynamic \cite{dynamic}, InputAware \cite{inputaware}, and WaNet \cite{wanet}. Their visual representations are shown in Fig. \ref{backdoor}. The accuracy on clean samples for these methods is 93.68 $\%$, 93.42 $\%$, 91.36 $\%$, and 92.19 $\%$, respectively, with attack success rates exceeding 99$\%$. For analysis, we sampled 500 instances per class, resulting in 5000 clean samples and 5000 backdoor-injected samples.

\textbf{Results.} As shown in Fig. \ref{backdoor}, the backdoor methods trigger different decision types. BadNet and Dynamic demonstrate key information sufficiency, allowing decisions consistent with the full input using only the backdoor data (right peak). In contrast, InputAware and WaNet activate the backdoor information but show significant decision variation (left peak).

\textbf{Analysis.}
The divergence in decision types stems from the construction of the backdoors. BadNet triggers the model to prioritize a unique feature absent in normal cases, enabling immediate classification upon backdoor activation. Dynamic retains this characteristic but introduces spatial diversity without changing the decision type. InputAware requires the backdoor feature is unreusable, then the impact of the backdoor feature on the decision changed significantly when the context changed. WaNet complicates this further by distorting inputs, making the decision reliant on pre-distorted content, rather than a new feature as in BadNet. Our results show that backdoor features can be categorized into two types: one that does not rely on contextual information to influence the decision, and one that does. As demonstrated in Section \ref{sec:5.1}, both feature types are present in clean samples. Therefore, traditional backdoor-based attribution evaluation methods, which rely on BadNet variants, only assess the first type of feature and fail to provide a complete evaluation. This limitation should be considered when evaluating backdoor performance in this domain.

\section{Conclusion}
In this paper, we present a unified theoretical analysis of commonly used modified BP attribution methods through the perspective of input alignment. We show that these attribution results inherently reflect the key input information utilized in decision-making. Our theory not only explains the existing perplexing behaviors, but also predicts phenomena beyond the scope of prior theories. Moreover, it enhances users' understanding of attribution results and guides their practical application. Experimental results reveal that the key information provided by attributions typically serves two roles: directly informing decisions or requiring contextual assistance. There is a significant correlation between the type of decision-making and the generalization capability of the model.
Using KIS, we reexamine features introduced by backdoor attacks and observe that features from a single attack generally play a consistent role, whereas features from different attacks vary depending on attack specifics. These findings highlight the limitations of current evaluations of backdoor attacks and offer new insights for improving their assessment.

\bibliography{main_new}
\bibliographystyle{ieee}

\vfill

\end{document}


\title{Supplementary Material for Unifying Perplexing Behaviors in Modified BP Attributions through Alignment Perspective} 

\author{Guanhua~Zheng,
        Jitao~Sang, 
        and~Changsheng~Xu, ~\IEEEmembership{Fellow,~IEEE,}
\thanks{Manuscript received XXX. The Associate Editor coordinating the review of this manuscript and approving it for publication was XXX. (Corresponding author: Changsheng Xu.)}
\IEEEcompsocitemizethanks{\IEEEcompsocthanksitem G. Zheng is with the University of Science and Technology of China, Hefei 230026, China (e-mail: zhenggh@mail.ustc.edu.cn).\protect\\
\IEEEcompsocthanksitem J. Sang is with the School of Computer and Information Technology and the Beijing Key Laboratory of Traffic Data Analysis and Mining, Beijing Jiaotong University, Beijing 100044, China (e-mail: jtsang@bjtu.edu.cn).\protect\\
\IEEEcompsocthanksitem C. Xu is with the National Lab of Pattern Recognition, Institute of Automation, CAS, Beijing 100190, China, and the University of Chinese Academy of Sciences (e-mail: csxu@nlpr.ia.ac.cn).}
}

\markboth{IEEE TRANSACTIONS ON NEURAL NETWORKS AND LEARNING SYSTEMS}%
{Shell \MakeLowercase{\textit{et al.}}: A Sample Article Using IEEEtran.cls for IEEE Journals}

\IEEEpubid{0000--0000/00\$00.00~\copyright~2021 IEEE}

\maketitle

\section*{A Proof of Proposition 1,2,3}
\begin{proposition}
the backpropagation rule of GBP:
\begin{equation}
r_{l-1}^g=W_l M_l \sigma(r^g_{l})
\end{equation}
can be formalized as a filtering rule for $F^g_l(W_l M_l)=W_l M_l diag(\mathbb{I}(r^g_{l}>0))$, and is negative filtering rule.
\end{proposition}
\begin{proof}
    The filtering rule of GBP is 
\begin{equation}
    F^g_l(W_l M_l)=W_l M_l diag(\mathbb{I}(r^g_{l}>0))
\end{equation}
 According to the definition of GBP, we can rewrite the filtering rule into an equivalent form:
\begin{equation}
    F^g_l(W_l M_l)=diag(\mathbb{I}(W_l M_l r^g_{l}>0))W_l M_l 
\end{equation}
Now we introduce a new dual rule of GBP, which zeros out the positive gradients at the target ReLU layer:
\begin{equation}
    F^{g-}_l(W_l M_l)=diag(\mathbb{I}(W_l M_l r^g_{l}<0))W_l M_l 
\end{equation}
Then we have:
\begin{equation}
\begin{split}
    &\langle \textbf{A}_{(l-1)}, W_l M_l r_l\rangle \\
    &=\langle \textbf{A}_{(l-1)}, F_l^g(W_l M_l) r_l\rangle+\langle \textbf{A}_{(l-1)}, F_l^{g-}(W_l M_l) r_l\rangle
\end{split}
\end{equation}
Note that $\textbf{A}_{(l-1)}$ is the output of a ReLU layer, so all terms of  $\textbf{A}_{(l-1)}$ are greater than or equal to 0. In addition, $F_l^{g-}$ zeros out all positive terms, so all the terms of $F_l^{g-}(W_l M_l) r_l$ is less than or equal to 0, so we have:
\begin{equation}
    \langle \textbf{A}_{(l-1)}, F_l^{g-}(W_l M_l) r_l\rangle\le 0
\end{equation}
The equals sign holds for $F_l^{g-}(W_l M_l) r_l=\textbf{0}$. Then 
\begin{equation}
\begin{split}
    &\langle \textbf{A}_{(l-1)}, F_l^g(W_l M_l) r_l\rangle\\
    &=\langle \textbf{A}_{(l-1)}, W_l M_l r_l\rangle-\langle \textbf{A}_{(l-1)}, F_l^{g-}(W_l M_l) r_l\rangle \\
    &> \langle \textbf{A}_{(l-1)}, W_l M_l r_l\rangle
\end{split}
\end{equation}
when $ F_l^g(W_l M_l)\neq W_l M_l$.
\end{proof}

\begin{proposition}
the backpropagation rule of RectGrad:
\begin{equation}
r_{l-1}^r=W_lM_l \left(r_{l}^r\odot\mathbb{I}(\textbf{A}_l \odot r_{l}^r>\tau) \right)
\end{equation}
can be formalized as a filtering rule for $F^r_l(W_l M_l)=W_l M_l diag(\mathbb{I}(\textbf{A}_l \odot r_{l}^r>\tau))$, and is negative filtering rule.
\end{proposition}
\begin{proof}
    The filtering rule of RectGrad is 
\begin{equation}
    F^r_l(W_l M_l)=W_l M_l diag(\mathbb{I}(\textbf{A}_l \odot r_{l}^r>\tau))
\end{equation}
According to the definition of RectGrad, we can rewrite the filtering rule in an equivalent form:
\begin{equation}
    F^r_l(W_l M_l)=diag(\mathbb{I}(\textbf{A}_{l-1} \odot W_l M_l r_{l}^r>\tau))W_l M_l 
\end{equation}
Now we introduce a new dual rule of RectGrad, which zeros out the positive gradients at the target ReLU layer:
\begin{equation}
    F^{r-}_l(W_l M_l)=diag(\mathbb{I}(\textbf{A}_{l-1} \odot W_l M_l r_{l}^r\le \tau))W_l M_l 
\end{equation}
Then we have:
\begin{equation}
\begin{split}
    &\langle \textbf{A}_{(l-1)}, W_l M_l r_l\rangle \\
    &=\langle \textbf{A}_{(l-1)}, F_l^r(W_l M_l) r_l\rangle+\langle \textbf{A}_{(l-1)}, F_l^{r-}(W_l M_l) r_l\rangle
\end{split}
\end{equation}
Note that $\textbf{A}_{(l-1)}$ is the output of a ReLU layer, so all terms of  $\textbf{A}_{(l-1)}$ are greater than or equal to 0. In addition, if $\tau\le0$ $F_l^{g-}$ zeros out all positive terms, so all the terms of $F_l^{r-}(W_l M_l) r_l$ is less than or equal to 0, we can obtain the same results as GBP.

If $\tau>0$, the results depend on the signal of $\langle \textbf{A}_{(l-1)}, F_l^{g-}(W_l M_l) r_l\rangle$, so if $\tau$ is large enough and the sum of positive terms is greater than the absolute value of the sum of negative terms, RectGrad is not an NFR. Practically, even $90\%$ mask rate will still be NFR according to our experiments in Section 3 and Section 5.
\end{proof}

\IEEEpubidadjcol
\begin{proposition}
The backpropagation rule of $z^+$rule:
\begin{equation}
r_{(l-1)_{[i]}}^{+}=\sum_j{\left (\frac{z_{l_{[ij]}}^+}{\sum_k{z_{l_{[kj]}}^+}} \right )r_{l_{[j]}}^{+}}
\end{equation}
can be formalized as a filtering rule for $F^{+}_l(W_l M_l) =\gamma W_ldiag(\mathbb{I}(W_{l}>0)) M_l $ with bottom process $r_0\odot \textbf{x}$, where $\gamma$ is a normalization term to keep $\sum_i{r_{(l-1)_{[i]}}^{+}}=\sum_j{r_{l_{[j]}}^{+}}$, and is negative filtering rule.
\end{proposition}
\begin{proof}
(1) Firstly, we are going to proof that $z$-rule is equivalent the feature-wise product of the input and the partial derivative:
\begin{equation}
    r_l^z=\textbf{A}_l \odot \frac{\partial f}{\partial \textbf{A}_l}
\end{equation}
Such results have been proven by Marco et al. Here we rewrite it in our formalization.
By definition, on the top layer $L$, for target label $k$:
\begin{equation}
    r_{L_{[k]}}^z=f_k=W_L^T\textbf{A}_{L-1}=\sum_i{\omega_{L_{[ik]}}a_{(L-1)_{[i]}}}
\end{equation}
and
\begin{equation}
    r_{L_{[j]}}^z=0, j\neq k
\end{equation}
Then, for the layer $L-1$:
\begin{equation}
\begin{split}
    r_{(L-1)_{[i]}}^z&=\sum_j{\left (\frac{z_{L_{[ij]}}}{\sum_i'{z_{L_{[i'j]}}} } \right )r_{L_{[j]}}^{z}}\\
    &=\sum_j{\left (\frac{\omega_{L_{[ij]}}a_{(L-1)_{[i]}}}{\sum_{i'}{\omega_{L_{[i'j]}}a_{(L-1)_{[k]}}} } \right )r_{L_{[j]}}^{z}}\\
    &=\left (\frac{\omega_{L_{[ik]}}a_{(L-1)_{[i]}}}{\sum_{i'}{\omega_{L_{[i'k]}}a_{(L-1)_{[i']}}} } \right )\sum_i{\omega_{L_{[ik]}}a_{(L-1)_{[i]}}}\\
    &=\omega_{L_{[ik]}}a_{(L-1)_{[i]}}\\
    &=\frac{\partial f_k}{\partial a_{(L-1)_{[i]}}}a_{(L-1)_{[i]}}
\end{split}
\end{equation}
For the inductive step we start from the hypothesis that for layer $l$:
\begin{equation}
     r_{l_{[i]}}^z=\frac{\partial f_k}{\partial a_{l_{[i]}}}a_{l_{[i]}}
\end{equation}
Then, for $l-1$ it holds:
\begin{equation}
\begin{split}
    r_{(l-1)_{[i]}}^z&=\sum_j{\frac{z_{l_{[ij]}}}{\sum_i'{z_{l_{[i'j]}}} } r_{l_{[j]}}^{z}}\\
    &=\sum_j{\frac{\omega_{l_{[ij]}}a_{(l-1)_{[i]}}}{\sum_{i'}{\omega_{l_{[i'j]}}a_{(l-1)_{[k]}}} } }\frac{\partial f_k}{\partial a_{l_{[j]}}}a_{l_{[j]}}\\
    &=\sum_j{\frac{a_{l_{[j]}}}{\sum_{i'}{\omega_{l_{[i'j]}}a_{(l-1)_{[k]}}} } }\frac{\partial f_k}{\partial a_{l_{[j]}}}\omega_{l_{[ij]}}a_{(l-1)_{[i]}}\\
    &=\sum_j{\frac{\partial f_k}{\partial a_{l_{[j]}}}\omega_{l_{[ij]}}a_{(l-1)_{[i]}}}\\
    &=\frac{\partial f_k}{\partial a_{(l-1)_{[i]}}}a_{(l-1)_{[i]}}
\end{split}
\end{equation}
The last equation is the chain-rule.

(2) Then, we will show the changes between $z$-rule and $z^+$-rule.
According to the forward propagation $\textbf{A}_{l}=\sigma(W_{l}^{T}\textbf{A}_{l-1})$, the ReLU layers make $\textbf{A}_l\ge 0$ for $l>0$. Note that $z_{l_{[ij]}}=\omega_{l_{[ij]}}a_{(l-1)_{[i]}}$, if $(l-1)$ is not the bottom layer, we have $a_{(l-1)_{[i]}}\ge 0$.

Therefore, we have
\begin{equation}
    z_{l_{[ij]}}^+=(z_{l_{[ij]}},0)_+=\omega_{l_{[ij]}}^+a_{(l-1)_{[i]}}
\end{equation}

From this equation, if we substitute the weight $W_l$ to $W_ldiag(\mathbb{I}(W_{l}>0))$ ,the backpropagation result just satisfies the $z^+$-rule. A noteworthy detail is that the normalization term $\sum_k{z_{l_{[kj]}}}$ can keep the raw $\sum{r_l^+}=\sum{\frac{\partial f}{\partial \textbf{A}_l}\odot \textbf{A}_l }$ unchanged, but if we simply use $F^{+}_l(W_l M_l)$ to backpropagate the result, it is clear that $\sum{r}$ will increase. So we introduce $\gamma$ to change the raw normalization terms to meet the constraints.

The core of the inequality in NFR definition is similar to GBP and RectGrad, so we will not go into details. It is noticed that all negative weights have been masked, so $z^+$ rule is NFR if and only if the original relevance, the output $f_k$ is positive, otherwise, $r^+$ will keep negative, thus mask negative weights will increase the target inner products.
\end{proof}

\section*{B Proof of Theorem 1}
\begin{theorem}
    In a random three-layer neural network where every entry is assumed to be independently isotropy distributed with a zero mean, if the number of filters $N$ is sufficiently large, the results of NFR can be approximated as:
    \begin{equation}
        R(\textbf{x})\approx \textbf{x}
    \end{equation}
\end{theorem}
\begin{proof}
In a random three-layer neural network 
\begin{equation}
f=V^T\sigma(W^T\textbf{x})=V^TMW^T\textbf{x}
\end{equation}
where $M=diag(\mathbb{I}(W^{T}\textbf{x}))$ denotes the gradient mask of the ReLU operation, $\mathbb{I}(\cdot)$ is the indicator function output 1 if input greater than 0, else output 0. Every entry of both $V$ and $W$ is assumed to be independently isotropy distributed with a zero mean, we have $V_{i}$ $M_i$ and  $W_i$ for $\forall i=1...N$. The probability density function is denoted as $p(\cdot)$.

Then the backpropagation results of NFR is :
\begin{equation}
\begin{split}
    R(\textbf{x})&=\frac{1}{Z}\sum_{i=1}^{N}h(V_i)M_iW_i \\
    &=\frac{1}{Z}\sum_{i=1}^{N}h(V_i)w^{(i)}
\end{split}
\label{expect}
\end{equation}
where $w^{(i)}=M_iW_i$, $Z$ is the normalization coefficient to ensure $|R|\in [0,1]$, $h(\cdot)$ is the negative filtering rule, e.g., GBP can zero out all $V_{i}<0$, so $h^g(V_i)=(V_i,0)_i$. 

Assuming the number of filters $N$ is sufficiently large, in order to ensure $|R|\in [0,1]$, we first set $Z=N\Tilde{Z}$, then from Eq.(\ref{expect}), we have
\begin{equation}
\label{expect2}
\begin{split}
    R(\textbf{x})&=\frac{1}{\Tilde{Z}}\frac{1}{N}\sum_{i=1}^{N}h(V_i)w^{(i)}\\
    &\overset{(a)}{\approx }\frac{1}{\Tilde{Z}}\mathbb{E}\left[h(V_i)w^{(i)}\right]\\
    &\overset{(b)}{=}\frac{1}{\Tilde{Z}}\mathbb{E}\left[h(V_i)\right] \mathbb{E}\left[w^{(i)}\right]
\end{split}
\end{equation}
where (a) follows from the asymptotic approximation of sample mean to the expectation and (b) follows from the fact that $h(V_i)$ and $w^{i}$ are independent.

According to the zero mean, $\mathbb{E}\left[V_i\right]=0$. Therefore, according to the definition of negative filtering rule and activation is nonnegative, we have $\mathbb{E}\left[h(V_i)\right]>0$ and is a constant, we denote it as $c_1$, so it will not affect the direction of $R$.

As for the $\mathbb{E}\left[w^{(i)}\right]=\mathbb{E}\left[M_iW_i\right]$, it is just half of the raw distribution of $W$ that all the $\left<W_i,\textbf{x} \right> <0$ have been masked. 

Then its expectation is given by
\begin{equation}
\begin{split}
    \mathbb{E}\left[w^{(i)}\right]&=\int_{\left<w,\textbf{x} \right> >0}{w\cdot p(w)dw}\\
    &\overset{(a)}{=}\int_{\phi_n>0}{U\phi\cdot p(\phi)|U|d\phi}\\
    &\overset{(b)}{=}U\int_{\phi_n>0}{\phi\cdot p(\phi)d\phi}\\
\end{split}
\label{eq32}
\end{equation}
Since the distribution $p(w)$ is isotropic, it is rotation invariance, where (a) follows the changes of variables $w=U\phi$, and $U$ is an unitary matrix satisfying the condition that $U^T\cdot \frac{\textbf{x}}{\left\| \textbf{x}\right\|}=e^{(n)}$ and $e^{(n)}$ is an unit vector with only the $n$-th entry being 1. That is,  $\frac{\textbf{x}}{\left\| \textbf{x}\right\|}$ is the $n$-th column of $U$. Thus, $\left<w,\textbf{x} \right>=\left<U\phi,\textbf{x} \right>=\phi^TU^T\textbf{x}=\phi^Te^{(n)}\left\| \textbf{x}\right\|=\phi_n^T\left\| \textbf{x}\right\|$, with $\phi_n$ is the $n$-th entry of $\phi$, which means $\left<w,\textbf{x} \right> >0$ is equivalent to $\phi_n>0$.

Also, by the change of variables in the integral, we have $dw=|U|d\phi$, where $|\cdot|$ denotes the determinant of a matrix. (b) follows from $|U|=1$ by the definition of an unitary matrix, and the swap between matrix multiplication and the integral.

As $\phi$ is a N-dimensional vector, the integral above can be evaluated at each entry, denoted by $\phi_m$, of $\phi$ separately. For $m \neq n$, we have

\begin{equation}
    \begin{split}
        &\int_{\phi_n>0}{\phi_m\cdot p(\phi)d\phi}\overset{(a)}{=}0
    \end{split}
\end{equation}
where (a) follows the zero integral property of odd functions as $\phi_m$ is orthogonal to the constraint $\phi_n>0$, so it is zero.

For $m=n$, we have 
\begin{equation}
    \begin{split}
        \int_{\phi_n>0}{\phi_p\cdot p(\phi)d\phi}\overset{(a)}{=}\int_{0}^{+\infty}{p(\phi_n)d\phi_n}=c_2>0
    \end{split}
\end{equation}
where (a) follows from the expansion of the multiple integral, and all of the other $N-1$ integrals over $\phi_k$ for
$k\neq n$ are 1, we use  $c_2$ to denote the positive constant.

Therefore, the Eq.(\ref{eq32}) becomes:
\begin{equation}
    \mathbb{E}\left[w^{(i)}\right]=c_2Ue^{(m)}\overset{(a)}{=}c_2 \frac{\textbf{x}}{\left\| \textbf{x}\right\|}
\end{equation}
where (a) follows from the the definition of the unitary matrix $U$ satisfying $U^T\cdot \frac{\textbf{x}}{\left\| \textbf{x}\right\|}=e^{(n)}$.
Therefore, putting all results together, we have:
\begin{equation}
\begin{split}
    R(\textbf{x})&=\frac{1}{\Tilde{Z}}\mathbb{E}\left[h(V_i)\right] \mathbb{E}\left[w^{(i)}\right]\\
    &=\frac{1}{\Tilde{Z}}c_1 c_2 \frac{\textbf{x}}{\left\| \textbf{x}\right\|}\\
    &=C\textbf{x}
\end{split}
\end{equation}
Thus, by setting the normalization coefficient $Z=\frac{Nc_1 c_2}{\left\| \textbf{x}\right\|}$,
we get the result.
\end{proof}

\section*{C Proof of Theorem 2}
\begin{theorem}
    \label{theorem2}
    In a three-layer neural network where any two of all the weights of the nerons in the hidden layer have the same $L_2-Norm$ are orthogonal to each other. and the activation is $A$. Suppose $r$ is the attribution backpropagate to the hidden layer, then consider two middle layer attribution $r_a$ and $r_b$, where $\alpha(r_a,A)=1$ , and every entry in $r_b$ is assumed to be i.i.d for some unknown distribution, then the final results of NFR of $r_a$ and $r_b$ have:
    \begin{equation}
        \alpha(R_a,\textbf{x})\ge \alpha(R_b,\textbf{x})
    \label{target}
    \end{equation}
    where equality holds if and only if all the $A_i$ is equal.
\end{theorem}

\begin{proof}
without loss of generality, we set $r_a=c\cdot A$, $\|W_i\|=1$, $w^{(i)}=M_iW_i$.
According to the Eqn.(\ref{expect}), we have 
\begin{equation}
    \begin{split}
        R_a&=\frac{1}{Z}\sum_{i=1}^{N}c\cdot A^{(i)}w^{(i)} \\
    \end{split}
\end{equation}

Then the alignment between $R_a$ and $\textbf{x}$ is 
\begin{equation}
    \begin{split}
        \alpha(R_a,\textbf{x})&=\frac{(\sum_{i=1}^{N}\cdot A^{(i)}w^{(i)})\cdot \textbf{x}}{\left \| \sum_{i=1}^{N}A^{(i)}w^{(i)}\right \|\left \|\textbf{x}\right \|}\\
        &=\frac{(\sum_{i=1}^{N}\cdot A^{(i)}w^{(i)}\cdot \textbf{x})}{\left \| \sum_{i=1}^{N}A^{(i)}w^{(i)}\right \|\left \|\textbf{x}\right \|}\\
        &=\frac{(\sum_{i=1}^{N}\cdot A^{(i)}\cdot A^{(i)})}{\left \| \sum_{i=1}^{N}A^{(i)}w^{(i)}\right \|\left \|\textbf{x}\right \|}\\
    \end{split}
\end{equation}

For $R_b$, as it is independent, so according to Eqn.(\ref{expect2}), we have 
\begin{equation}
    \begin{split}
        \alpha(R_b,\textbf{x})&=\frac{(\sum_{i=1}^{N}w^{(i)})\cdot \textbf{x}}{\left \| \sum_{i=1}^{N}w^{(i)}\right \|\left \|\textbf{x}\right \|}\\
        &=\frac{(\sum_{i=1}^{N}w^{(i)}\cdot \textbf{x})}{\left \| \sum_{i=1}^{N}w^{(i)}\right \|\left \|\textbf{x}\right \|}\\
        &=\frac{(\sum_{i=1}^{N}
        A^{(i)})}{\left \| \sum_{i=1}^{N}w^{(i)}\right \|\left \|\textbf{x}\right \|}\\
    \end{split}
\end{equation}

Therefore, let $z_i=A^{(i)}$, $w_i=w^{(i)}$, Eqn.(\ref{target}) can be rewrite as:
\begin{equation}
    \begin{split}
        \frac{\sum_{i=1}^{N} z_i^2}{\left \| \sum_{i=1}^{N}{z_iw_i}\right \|}\ge \frac{\sum_{i=1}^{N}z_i}{\left \| \sum_{i=1}^{N}w_i\right \|}
    \end{split}
\end{equation}
Squaring both sides of the inequality:
\begin{equation}
    \begin{split}
        \frac{(\sum_{i=1}^{N} z_i^2)^2}{\left \| \sum_{i=1}^{N}{z_iw_i}\right \|^2}\ge \frac{(\sum_{i=1}^{N}z_i)^2}{\left \| \sum_{i=1}^{N}w_i\right \|^2}
    \end{split}
    \label{3}
\end{equation}

Not that any two weights are orthognal to each other, the denominators can be rewrite as:
\begin{equation}
    \left \| \sum_{i=1}^{N}{z_iw_i}\right \|^2=\sum_{i=1}^{N}{z_i^2 \left \| w_i \right \|^2}
\label{1}
\end{equation}
\begin{equation}
    \left \| \sum_{i=1}^{N}{w_i}\right \|^2=\sum_{i=1}^{N}{\left \| w_i \right \|^2}
\label{2}
\end{equation}

Substituting the Eqn. (\ref{1}) and (\ref{2}) into Eqn. (\ref{3})
\begin{equation}
    \begin{split}
        \frac{(\sum_{i=1}^{N} z_i^2)^2}{\sum_{i=1}^{N}{z_i^2 \left \| w_i \right \|^2}}\ge \frac{(\sum_{i=1}^{N}z_i)^2}{\sum_{i=1}^{N}{\left \| w_i \right \|^2}}
    \end{split}
\end{equation}

Simplifying the above inequality, we obtain:
\begin{equation}
    \begin{split}
        (\sum_{i=1}^{N} z_i^2)^2(\sum_{i=1}^{N}{\left \| w_i \right \|^2})\ge (\sum_{i=1}^{N}z_i)^2(\sum_{i=1}^{N}{z_i^2 \left \| w_i \right \|^2})
    \end{split}
\end{equation}
Note that $(\frac{\sum_{i=1}^{N}z_i^2}{N})\ge(\frac{\sum_{i=1}^{N}z_i}{N})^2$, where equality holds if and only if  $z_i$ is equal.

Therefore we only need to prove:
\begin{equation}
    \begin{split}
        (\sum_{i=1}^{N} z_i^2)^2(\sum_{i=1}^{N}{\left \| w_i \right \|^2})\ge (N\sum_{i=1}^{N}z_i^2(\sum_{i=1}^{N}{z_i^2 \left \| w_i \right \|^2})
    \end{split}
\end{equation}

Simplifying the above, we obtain:
\begin{equation}
    \begin{split}
        (\sum_{i=1}^{N} z_i^2)(\sum_{i=1}^{N}{\left \| w_i \right \|^2})\ge N(\sum_{i=1}^{N}{z_i^2 \left \| w_i \right \|^2}
    \end{split}
\end{equation}

As all $\left \| w_i \right \|^2$ are equal, the two sides of the above inequality are equal.

\end{proof}

\section*{C Architecture of VGG16}
We use the VGG16 from torchvision.models, and it is printed as follows:

{
\scriptsize
VGG(

\quad  (features): Sequential(
  
\qquad  (0): Conv2d(3, 64, kernelsize=(3, 3), stride=(1, 1), padding=(1, 1))

\qquad    (1): ReLU(inplace=True)
    
\qquad    (2): Conv2d(64, 64, kernelsize=(3, 3), stride=(1, 1), padding=(1, 1))
    
\qquad    (3): ReLU(inplace=True)
    
\qquad    (4): MaxPool2d(kernelsize=2, stride=2, padding=0, dilation=1, ceilmode=False)
    
\qquad    (5): Conv2d(64, 128, kernelsize=(3, 3), stride=(1, 1), padding=(1, 1))
    
\qquad    (6): ReLU(inplace=True)
    
\qquad   (7): Conv2d(128, 128, kernelsize=(3, 3), stride=(1, 1), padding=(1, 1))
    
\qquad    (8): ReLU(inplace=True)
    
\qquad    (9): MaxPool2d(kernelsize=2, stride=2, padding=0, dilation=1, ceilmode=False)
    
\qquad    (10): Conv2d(128, 256, kernelsize=(3, 3), stride=(1, 1), padding=(1, 1))
    
\qquad    (11): ReLU(inplace=True)
    
\qquad    (12): Conv2d(256, 256, kernelsize=(3, 3), stride=(1, 1), padding=(1, 1))
    
\qquad    (13): ReLU(inplace=True)
    
\qquad    (14): Conv2d(256, 256, kernelsize=(3, 3), stride=(1, 1), padding=(1, 1))
    
\qquad    (15): ReLU(inplace=True)
    
\qquad    (16): MaxPool2d(kernelsize=2, stride=2, padding=0, dilation=1, ceilmode=False)
    
\qquad    (17): Conv2d(256, 512, kernelsize=(3, 3), stride=(1, 1), padding=(1, 1))
    
\qquad    (18): ReLU(inplace=True)
    
\qquad    (19): Conv2d(512, 512, kernelsize=(3, 3), stride=(1, 1), padding=(1, 1))
    
\qquad    (20): ReLU(inplace=True)
    
\qquad    (21): Conv2d(512, 512, kernelsize=(3, 3), stride=(1, 1), padding=(1, 1))
    
\qquad    (22): ReLU(inplace=True)
    
\qquad    (23): MaxPool2d(kernelsize=2, stride=2, padding=0, dilation=1, ceilmode=False)
    
\qquad    (24): Conv2d(512, 512, kernelsize=(3, 3), stride=(1, 1), padding=(1, 1))
    
\qquad    (25): ReLU(inplace=True)
    
\qquad    (26): Conv2d(512, 512, kernelsize=(3, 3), stride=(1, 1), padding=(1, 1))
    
\qquad    (27): ReLU(inplace=True)
    
\qquad    (28): Conv2d(512, 512, kernelsize=(3, 3), stride=(1, 1), padding=(1, 1))
    
\qquad    (29): ReLU(inplace=True)
    
\qquad    (30): MaxPool2d(kernelsize=2, stride=2, padding=0, dilation=1, ceilmode=False)

\quad  )
 
\quad  (avgpool): AdaptiveAvgPool2d(outputsize=(7, 7))
  
\quad (classifier): Sequential(
 
\qquad    (0): Linear(infeatures=25088, outfeatures=4096, bias=True)
    
\qquad    (1): ReLU(inplace=True)
    
\qquad    (2): Dropout(p=0.5, inplace=False)
    
\qquad    (3): Linear(infeatures=4096, outfeatures=4096, bias=True)
    
\qquad    (4): ReLU(inplace=True)
    
\qquad    (5): Dropout(p=0.5, inplace=False)
    
\qquad    (6): Linear(infeatures=4096, outfeatures=1000, bias=True)
  
\quad  )

)
}
